\documentclass[letterpaper]{article} 
\usepackage{aaai2026}  
\usepackage{times}  
\usepackage{helvet}  
\usepackage{courier}  
\usepackage[hyphens]{url}  
\usepackage{graphicx} 
\usepackage{natbib}  
\usepackage{caption} 
\usepackage{algorithm}
\usepackage{algorithmic}

\usepackage{newfloat}
\usepackage{listings}
\DeclareCaptionStyle{ruled}{labelfont=normalfont,labelsep=colon,strut=off} 
\floatstyle{ruled}
\newfloat{listing}{tb}{lst}{}
\floatname{listing}{Listing}
\nocopyright

\def\equationautorefname~#1\null{Equation~(#1)\null}
\usepackage{url}
\usepackage{bm}
\usepackage{xspace}
\usepackage{algorithmic}
\usepackage{algorithm}
\usepackage{mdwlist}    
\usepackage{paralist} 

\usepackage{adjustbox}
\usepackage{array}
\usepackage{booktabs}
\usepackage{multirow}

\usepackage{array,graphicx}
\usepackage{booktabs}
\usepackage{pifont}
\usepackage{color}

\usepackage{colortbl}
\definecolor{lightgray}{gray}{0.85}
\usepackage{multirow}

\usepackage{fancybox} 
\usepackage{amsmath}
\usepackage{amsfonts}

\usepackage{enumitem}

\newcommand{\alg}[1]{Algorithm~#1}
\newcommand{\step}[1]{Line~#1}
\newcommand{\eq}[1]{Equation~(#1)}
\newcommand{\tabl}[1]{Table~#1}
\newcommand{\fig}[1]{Figure~#1}

\newtheorem{prop}{Proposition}
\newtheorem{prf}{Proof}

\newcommand{\hide}[1]{}

\newcommand{\method}{AnomalyFilter\xspace}

\newcommand{\mts}{\bm{X}} 
\newcommand{\mtsdiff}{\bm{X}} 
\newcommand{\mtsvector}{\bm{x}} 
\newcommand{\mtsstep}[1]{\mtsdiff_{#1}} 

\newcommand{\mtsreconst}{\hat{\mts}} 
\newcommand{\mtsvectorreconst}{\hat{\mtsvector}} 
\newcommand{\mtsstepreconst}[1]{\hat{\mtsdiff}_{#1}} 

\newcommand{\trainingset}{\mathcal{D}^{train}} 
\newcommand{\testset}{\mathcal{D}^{test}} 
\newcommand{\scorefunction}[1]{\mathcal{A}(#1)} 

\newcommand{\timestep}{L} 
\newcommand{\ndim}{K} 
\newcommand{\latentdim}{C} 

\newcommand{\net}{\theta} 

\newcommand{\forwarddiffstep}{T} 
\newcommand{\reversediffstep}{\lambda} 
\newcommand{\diffstep}{t} 

\newcommand{\varschedule}{\beta} 
\newcommand{\varschedulet}[1]{\varschedule_{#1}} 
\newcommand{\varscheduletilde}{\tilde{\varschedule}} 
\newcommand{\varscheduletildet}[1]{\varscheduletilde_{#1}} 

\newcommand{\alphaschedule}{\alpha} 
\newcommand{\alphaschedulet}[1]{\alphaschedule_{#1}} 
\newcommand{\alphaschedulebar}{\bar{\alphaschedule}} 
\newcommand{\alphaschedulebart}[1]{\alphaschedulebar_{#1}} 

\newcommand{\error}{\epsilon} 
\newcommand{\errort}[1]{\error_{#1}} 
\newcommand{\errornetsolo}{\error_{\net}} 
\newcommand{\errornet}[1]{\errornetsolo(#1)} 

\newcommand{\gaussnoise}{\zeta} 

\newcommand{\mask}{\bm{B}} 
\newcommand{\bernoulliparam}{p} 

\newcommand{\mean}{\mu} 
\newcommand{\meannetsolo}{\mean_{\net}} 
\newcommand{\meannet}[1]{\meannetsolo(#1)} 

\newcommand{\lossfunc}{\mathcal{L}} 
\newcommand{\lossfuncmask}{\lossfunc_{Mask}} 
\newcommand{\lossfuncnonmask}{\lossfunc_{NonMask}} 
\newcommand{\lossweight}{c} 
\newcommand{\noiseweight}{e} 

\newcommand{\idmat}{\bm{I}} 

\newcommand{\realnumber}{\mathbb{R}} 

\newcommand{\gauss}[1]{\mathcal{N}(#1)} 
\newcommand{\forward}[1]{q(#1)} 
\newcommand{\conditionforward}[2]{q(#1\:|\:#2)}

\newcommand{\conditionreversenet}[2]{p_{\net}(#1\:|\:#2)}

\newcommand{\expectation}[1]{\mathbb{E}[#1]}
\newcommand{\expectationunder}[2]{\mathbb{E}_{#2}[#1]}

\newcommand{\hadamard}{\circ} 

\newcommand{\smd}{SMD\xspace}

\newcommand{\at}{Anomaly-Transformer\xspace}
\newcommand{\beatgan}{BeatGAN\xspace}
\newcommand{\isf}{IsolationForest\xspace}
\newcommand{\ocsvm}{OCSVM\xspace}

\newcommand{\lstmvae}{LSTM-VAE\xspace}
\newcommand{\tranad}{TranAD\xspace}
\newcommand{\usad}{USAD\xspace}
\newcommand{\imdiffusion}{IMDiffusion\xspace}
\newcommand{\dddr}{D3R\xspace}
\newcommand{\diffad}{DiffAD\xspace}
\newcommand{\dada}{DADA\xspace}
\newcommand{\dae}{DAE\xspace}
\newcommand{\ddpm}{DDPM\xspace}

\title{
Selective Denoising Diffusion Model for Time Series Anomaly Detection
}
\author {
    Kohei Obata\textsuperscript{\rm 1}\thanks{Corresponding author.},
    Zheng Chen\textsuperscript{\rm 1},
    Yasuko Matsubara\textsuperscript{\rm 1},
    Lingwei Zhu\textsuperscript{\rm 2},
    Yasushi Sakurai\textsuperscript{\rm 1}
}
\affiliations {
    \textsuperscript{\rm 1}SANKEN, The University of Osaka, Osaka, Japan\\
    \textsuperscript{\rm 2}Great Bay University, Dongguan, China\\
    \{obata88, chenz, yasuko, yasushi\}@sanken.osaka-u.ac.jp, zhulingwei@gbu.edu.cn
}

\usepackage{bibentry}

\begin{document}

\maketitle

\begin{abstract}
    Time series anomaly detection (TSAD) has been an important area of research for decades,
with reconstruction-based methods, mostly based on generative models, gaining popularity and demonstrating success.
Diffusion models have recently attracted attention due to their advanced generative capabilities.
Existing diffusion-based methods for TSAD rely on a conditional strategy, which reconstructs input instances from white noise with the aid of the conditioner. 
However, this poses challenges in accurately reconstructing the normal parts, resulting in suboptimal detection performance. 
In response, we propose a novel diffusion-based method, named \method, which acts as a selective filter that only denoises anomaly parts in the instance while retaining normal parts.
To build such a filter, we mask Gaussian noise during the training phase and conduct the denoising process without adding noise to the instances.
The synergy of the two simple components greatly enhances the performance of naive diffusion models.
Extensive experiments on five datasets demonstrate that \method achieves notably low reconstruction error on normal parts, providing empirical support for its effectiveness in anomaly detection.
\method represents a pioneering approach that focuses on the noise design of diffusion models specifically tailored for TSAD.  
\end{abstract}

\begin{links}
    \link{Code}{https://github.com/KoheiObata/AnomalyFilter}
\end{links}

\section{Introduction}
    \label{010intro}

\begin{figure*}[t]
    \centering
    \includegraphics[width=1.0\linewidth]{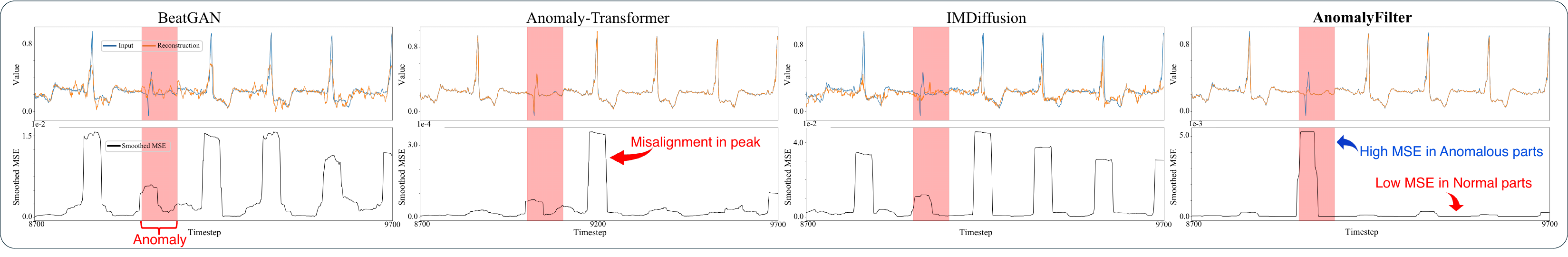} \\
    \caption{
    Reconstruction examples
    of encoder-decoder models (i.e., \beatgan and \at), conditional diffusion models (i.e., \imdiffusion), and proposed \method.
    }
    \label{fig:comparison}
\end{figure*}

Time series anomaly detection (TSAD) is an important task for real-world applications, including robot-assisted systems~\cite{lstmvae,robot}, engines~\cite{engine_lstmencdec}, and cyber-physical system maintenance~\cite{hs-tree}.
The goal of TSAD is to detect unexpected/unusual patterns/values within sequences that are significantly dissimilar to the majority.
Since anomalies are typically rare and labeling them from quantities of data is laborious,
TSAD is generally approached as an unsupervised learning problem in which detection relies on a normality assumption, namely, most observations are normal.

A prevalent approach for TSAD involves reconstruction-based methods. 
These methods learn to reconstruct normal samples during training and assume that normal samples are better reconstructed than anomalous samples during inference.
Then, anomalies can be detected by thresholding the difference between the input and its reconstruction.
Hence, the ideal reconstruction exhibits low reconstruction error on normal parts and high error on anomalous parts.
Many methods, such as those based on GAN~\cite{beatgan,tanogan} and Transformer~\cite{anomalytransformer,tranad}, are categorized as encoder-decoder models because they reconstruct samples from embeddings in latent space.
However, controlling the balance between concise representation (i.e., regularization) and representation power (i.e., reconstruction) remains challenging~\cite{memory}.
This issue is unavoidable when reconstructing from latent space.

Recently, diffusion models~\cite{ddpm} have emerged as a promising approach due to their advanced generative capabilities~\cite{text2image}.
Unlike encoder-decoder models, diffusion models operate directly in data space, gradually reconstructing data from noisy inputs through an iterative denoising process.
Current diffusion models for TSAD primarily focus on the creation of an effective conditional strategy of introducing inductive bias into the denoising process~\cite{diffad, imdiffusion}.
However, reconstructing from white noise proves challenging even with the aid of the conditioner due to the unavailability of original information.
Consequently, while it denoises anomalous parts well and the reconstruction of anomalous parts deviates from the original input, it limits the accurate reconstruction of the normal parts.

Thus, a natural research question arises: \textit{How can we accurately reconstruct the normal parts while exploiting the denoising capability of diffusion models?}
In this work, we propose \method, a new diffusion-based method designed for TSAD, which overcomes the limitation of the existing diffusion-based methods regarding reconstruction quality.
Specifically, we build a selective filter that denoises only the anomalous parts in the sample while retaining the normal parts.
While alternative noise strategies for diffusion models have been explored in the image domain~\cite{anoddpm,nongaussian}, the impact of noise design on time series remains largely underexplored.
Thus, we investigate the utility of \textit{masked Gaussian noise}, which is the noise generated through a combination of Gaussian noise and Bernoulli mask.
During the training, the model learn to predict the noise added to a sample similar to general diffusion models, enabling the model to act as a filter that selectively denoises nonmasked parts (i.e., erases Gaussian noise) while passing through masked parts (i.e., retains input).
Furthermore, we observe that adding noise to the sample during the denoising process merely degrades reconstruction quality, leading to poor detection accuracy.
To address this, we propose \textit{noiseless inference}, which inputs a scaled noiseless sample into the model and conducts the denoising process without any noise, aligning seamlessly with the masked Gaussian noise.
Neither masked Gaussian noise nor noiseless inference alone contributes significantly to improving the performance of the vanilla DDPM.
However, the synergy of the two simple components allows the model to act as an effective filter, achieving substantial performance improvements over the vanilla DDPM.

\subsubsection{Results Overview}
\fig{\ref{fig:comparison}} shows the reconstruction results and reconstruction error (i.e., MSE) of four methods.
Remind that the y-axis units differ across methods.
\beatgan prioritizes concise expression and struggles at reconstructing peaks, resulting in consistently high MSE.
\at, on the other hand, has strong representation power, causing it to accurately reconstruct even the anomalous parts, so-called ``identical shortcut'' issue~\cite{unified}.
The relatively high MSE around 9200 is caused by a slight misalignment in reconstructing the peak.
These results suggest that encoder-decoder models face challenges in balancing concise representation and representation power.
\imdiffusion denoises anomalous parts well but fails to reconstruct normal parts, especially peaks, which demonstrates the limitation of conditional diffusion models.
As a result, these methods fail to detect the anomalous parts.
Our proposed \method shows a significantly higher MSE only in the anomalous parts while performing reconstruction with low MSE for the remaining parts, making anomaly detection easier.

\subsubsection{Contributions}
In summary, this paper makes the following contributions:
\begin{itemize}[left=0pt]
    \item We investigate the limitations of the existing diffusion-based methods for TSAD from the perspective of noise design and empirically demonstrate their impact on reconstruction quality and anomaly detection performance.
    
    \item We introduce \method
    , a novel diffusion-based method designed for TSAD.
    It achieves ideal reconstruction, i.e., a low reconstruction error for normal parts and relatively high reconstruction error for anomalous parts.
    
    \item We show the superiority of \method by comparing it with state-of-the-art baselines on five datasets.
    Notably, by combining two simple components, masked Gaussian noise and noiseless inference, we achieve a $45.1\%$ improvement in VUS-PR compared with vanilla DDPM.
\end{itemize}

\section{Related Work}
    \label{020related}

\subsubsection{Diffusion Model in Time Series}
Diffusion models are a type of generative model that gains the ability to generate diverse samples by corrupting training samples with noise and learning to reverse the process~\cite{survey_ddpm,survey_ddpm3}.
They have attracted increasing attention in the field of AI-generated content, such as text-to-image synthesis~\cite{text2image} and audio synthesis~\cite{diffwave,wavegrad}.
As regards time series, conditional diffusion models have been applied to various time series tasks such as imputation~\cite{midm,sasdim,lscd} and forecasting~\cite{sssd,ldt,tmdm,armd}.
They utilize diffusion models conditioned on observed values to impute/predict unseen values, achieving good success.
For instance, CSDI~\cite{csdi} is designed for time series imputation, and it employs feature and temporal attention sequentially to learn the noise at each step while introducing conditions in the denoising process.
TimeGrad~\cite{timegrad} is a conditional diffusion model that predicts in an autoregressive manner, with the denoising process guided by the hidden state of a recurrent neural network. 
Additionally, diffusion models have been employed for time series data augmentation~\cite{fts-diffusion,timeweaver,tsdiff,padts}.

\subsubsection{Diffusion Model for Time Series Anomaly Detection}
TSAD differs from the aforementioned tasks in that all values are known.
Overly powerful conditioners risk carrying anomaly information, which can lead to the accurate reconstruction of even anomalous parts.
Therefore, the current mainstream focus is on a unique conditional strategy~\cite{diffusionae,d3r,modem}.
For example, \imdiffusion~\cite{imdiffusion} is inspired by imputation techniques, where part of the sample is masked, and the unmasked parts serve as a condition to impute the masked parts.
\diffad~\cite{diffad} tackles the case where anomalous points concentrate in certain parts, which is called anomaly concentration, and incorporates a density ratio-based point selection strategy for masking anomalous parts.
However, since they rely on partially observed input as a condition, which often loses information needed for reconstructing normal parts, their reconstruction errors are relatively high, resulting in a failure to detect anomalies that differ slightly from normalities (e.g., pattern-wise outliers).

\subsubsection{Denoising Model for Anomaly Detection}

Denoising models, including diffusion models, perform anomaly detection by taking corrupted inputs, learning to denoise them, and generating clean reconstructions.
Similar to diffusion models, a denoising autoencoder (DAE)~\cite{dae} adds noise (typically Gaussian) to the input but differs in that it directly predicts the clean input rather than using a stepwise denoising process.
The selection of noise is critical in denoising models.
Recent studies have explored deterministic noise~\cite{colddiffusion,softdiffusion}, and non-Gaussian noise~\cite{nongaussian} for diffusion models.
In image anomaly detection, \cite{roleofnoise} shows that DAE trained with coarse noise outperforms diffusion models at Brain MRI, and AnoDDPM~\cite{anoddpm} demonstrates the effectiveness of Simplex noise.
Despite these advancements in the image domain, where noise has been extensively studied and often operates on 2D data, noise tailored for time series remains largely unexplored.
This highlights a significant gap in understanding the impact of noise design in TSAD applications.

\section{Preliminaries}
    \label{030preliminary}
    
\subsection{Problem Definition} \label{sec:problem}
In this paper, we focus on unsupervised anomaly detection for time series.
We adopt a local contextual window to model their temporal dependency, which is a common practice with most deep TSAD methods.
Consider an instance $\mts = \{ \mtsvector_1,\mtsvector_2,\dots,\mtsvector_\timestep \} \in \realnumber^{\ndim \times \timestep}$ to be a multivariate time series with $\ndim$ features and $\timestep$ timesteps.
The objective of TSAD is to identify anomalous timesteps in a test set $\testset$ by training the model in an unlabeled training set $\trainingset$ assumed to contain mostly normal instances.
Thus, the model aims to assign an anomaly score to each timestep, where higher scores are more anomalous.
In a reconstruction-based approach, the model outputs a reconstruction $\mtsreconst$ of the given instance $\mts$, and the anomaly score is calculated by comparing them by a score function $\scorefunction{\mtsvector_{l}, \mtsvectorreconst_{l}} : \realnumber^{\ndim} \to \realnumber$ (for simplicity, we consider the MSE score).
To solve TSAD efficiently, we need a powerful method that achieves ideal reconstruction, i.e., has a low reconstruction error on the normal parts and a high reconstruction error on the anomalous parts of the instance.

\subsection{Denoising Diffusion Probablistic Model} \label{sec:ddpm} 
Our \method is based on the denoising diffusion probabilistic model (DDPM)~\cite{ddpm}, a well-known diffusion model.
DDPM consists of a forward diffusion process and a reverse denoising process.
The forward diffusion process
$\conditionforward{\mtsstep{\diffstep}}{\mtsstep{\diffstep-1}}$
gradually corrupts data from some target distribution
$\forward{\mtsstep{0}}$
into a normal distribution,
where $\mtsstep{\diffstep}$ denotes corrupted data indexed by a diffusion step $\diffstep$.
Thus, in practice, $\mtsstep{0}$ is $\mts$.
Then, a learned reverse process
$\conditionreversenet{\mtsstep{\diffstep-1}}{\mtsstep{\diffstep}}$
generates data by turning noise into data from 
$\forward{\mtsstep{0}}$.
%

\subsubsection{Forward Process}
This process can be formally described as Markovian:
\begin{align}
    \conditionforward{\mtsstep{\diffstep}}{\mtsstep{\diffstep-1}} =
    \gauss{\mtsstep{\diffstep} ; \sqrt{1-\varschedulet{\diffstep}} \mtsstep{\diffstep-1} , \varschedulet{\diffstep}\idmat}, \label{eq:forwardprocess}
\end{align}
for $\diffstep=\{ 1,\dots,\forwarddiffstep \}$, where $\forwarddiffstep$ is a forward diffusion step.
The variance schedule $\varschedulet{\diffstep} \in (0,1)$ is a hyperparameter defined by a monotonically increasing sequence $\varschedulet{1} < \varschedulet{2} < \dots < \varschedulet{\forwarddiffstep}$ such that the data become sequentially corrupted.
%

To sample from an arbitrary $\diffstep$,
namely without having to find intermediate steps $\mtsstep{\diffstep-1}, \dots, \mtsstep{1}$,
the variance schedule is simplified to 
$\alphaschedulet{\diffstep} = 1-\varschedulet{\diffstep}$
and 
$\alphaschedulebart{\diffstep} = \prod_{i=1}^{\forwarddiffstep} \alphaschedulet{i}$, as follows:
\begin{align}
    \conditionforward{\mtsstep{\diffstep}}{\mtsstep{0}} &=
    \gauss{\mtsstep{\diffstep} ; \sqrt{\alphaschedulebart{\diffstep}}\mtsstep{0} , (1-\alphaschedulebart{\diffstep})\idmat},
    \label{eq:fastforward} \\
    \mtsstep{\diffstep} &= 
    \alphaschedulebart{\diffstep}\mtsstep{0} + 
    \sqrt{1-\alphaschedulebart{\diffstep}}\errort{\diffstep},
    \label{eq:fastforwardx}
\end{align}
where $\errort{\diffstep} \sim \gauss{\mathbf{0}, \idmat}$ is the same size as $\mtsstep{0}$.
%

\subsubsection{Reverse Process}
The generative model, parameterized by $\net$, is a learned reverse process to denoise 
$\mtsstep{\forwarddiffstep}$ and reconstruct $\mtsstep{0}$.
This is accomplished by iteratively computing the following Gaussian transitions:
\begin{align}
    \conditionreversenet{\mtsstep{\diffstep-1}}{\mtsstep{\diffstep}} = 
    \gauss{\mtsstep{\diffstep-1} ; \meannet{\mtsstep{\diffstep},\diffstep} , \varscheduletildet{\diffstep}\idmat}, \label{eq:reverseprocess}
\end{align}
for $\diffstep=\{ \forwarddiffstep,\dots,1 \}$,
where $\varscheduletildet{\diffstep} = \frac{1-\alphaschedulebart{\diffstep-1}}{1-\alphaschedulebart{\diffstep}}$.
$\meannetsolo$ can, for example, be implemented with U-net-like architectures~\cite{unet},
and it can be learned by setting
\begin{align}
    \meannet{\mtsstep{\diffstep} , \diffstep} = 
    \frac{1}{\alphaschedulet{\diffstep}}
    \big(
        \mtsstep{\diffstep} - 
        \frac{\varschedulet{\diffstep}}{\sqrt{1-\alphaschedulebart{\diffstep}}} \errornet{\mtsstep{\diffstep} , \diffstep}
    \big).
    \label{eq:meannetwork}
\end{align}

\subsubsection{Objective Function}
\cite{ddpm} shows that optimizing the following simplified training objective leads to better generation quality: 
\begin{align}
    \lossfunc = 
    \expectationunder{\lVert \errort{\diffstep} - \errornet{\mtsstep{\diffstep} , \diffstep} \rVert}{\mtsstep{0} \sim \forward{\mtsstep{0}}, \errort{\diffstep} \sim \gauss{\mathbf{0}, \idmat}, \diffstep}, 
    \label{eq:loss}
\end{align}
where $\errort{\diffstep}$ is the noise used to obtain $\mtsstep{\diffstep}$ from $\mtsstep{0}$ in \eq{\ref{eq:fastforwardx}} at step $\diffstep$,
and $\lVert \cdot \rVert$ denotes the Frobenius norm.

\begin{algorithm}[t]
    \caption{Naive Training}
    \label{alg:training}
    \begin{algorithmic}[1]
        \STATE {\bf Input:}
        Training set $\trainingset$,
        Forward diffusion step $\forwarddiffstep$;
        %
        \REPEAT
            \STATE Sample data
            $\mtsstep{0} \sim \trainingset$; \label{step:trainst}
            \STATE Sample diffusion step
            $\diffstep \sim Uniform(\{ 1, \dots, \forwarddiffstep \})$;
            \STATE Sample noise
            $\errort{\diffstep} \sim \gauss{\mathbf{0}, \idmat}$;
            \STATE Corrupt 
            $\mtsstep{\diffstep} = 
            \alphaschedulebart{\diffstep}\mtsstep{0} + 
            \sqrt{1-\alphaschedulebart{\diffstep}}\errort{\diffstep}$;
            \STATE Calculate
            $\lossfunc =
            \lVert \errort{\diffstep} - \errornet{\mtsstep{\diffstep} ,\diffstep} \rVert$; \label{step:trained}
        \UNTIL{Converged;}
    \end{algorithmic}
    \normalsize
\end{algorithm}

\begin{algorithm}[t]
    \caption{Naive Inference}
    \label{alg:naiveinference}
    \begin{algorithmic}[1]
        \STATE {\bf Input:}
        Data $\mtsstep{0} \in \testset$,
        Reverse diffusion step $\reversediffstep$,
        Trained denoising function $\errornetsolo$;
        \STATE {\bf Output:}
        Reconstruction $\mtsstepreconst{0}$ (i.e., $\mtsreconst$);
        \STATE Corrupt $\mtsstepreconst{\reversediffstep} =
        \alphaschedulebart{\reversediffstep}\mtsstep{0} +
        \sqrt{1-\alphaschedulebart{\reversediffstep}}\errort{\reversediffstep}$; \label{step:naiveinfcorrupt}
        \FOR{$\diffstep=\reversediffstep:1$}
            \STATE Sample
            $\errort{\diffstep} \sim \gauss{\mathbf{0}, \idmat}$ if $\diffstep > 1$ else $\errort{\diffstep} = 0$;
            \STATE Denoise
            $\mtsstepreconst{\diffstep-1} =
            \frac{1}{\sqrt{\alphaschedulet{\diffstep}}}
            \big(
            \mtsstepreconst{\diffstep} - 
            \frac{\varschedulet{\diffstep}}{\sqrt{1-\alphaschedulebart{\diffstep}}} \errornet{\mtsstepreconst{\diffstep} , \diffstep}
            \big) +
            \varscheduletildet{\diffstep}\errort{\diffstep}$; \label{step:naiveinfdenoise}
        \vspace{-1em}
        \ENDFOR
        \STATE \textbf{return} $\mtsstepreconst{0}$;
    \end{algorithmic}
    \normalsize
\end{algorithm}

\subsection{Anomaly Detection with Diffusion Model} \label{sec:ddpm4ad} 
Here, we focus on general anomaly detection tasks that are not restricted to time series.
The diffusion model is a reconstruction-based approach.
Thus, hereafter, we describe the way in which the diffusion model reconstructs $\mtsreconst$.

\subsubsection{Training Phase}
During the training phase, the model $\errornetsolo$ learns to predict the noise $\errort{\diffstep}$ added to $\mtsstep{0} \in \trainingset$ given $\mtsstep{\diffstep}$ and $\diffstep$.
The training procedure is shown as \alg{\ref{alg:training}}.
In each iteration of the training (\step{\ref{step:trainst}}-\ref{step:trained}),
a diffusion step $\diffstep$ is randomly sampled, at which point the diffusion process \eq{\ref{eq:fastforwardx}} is applied.
Unlike the other tasks, in anomaly detection, $\mtsstep{\forwarddiffstep}$ does not need to be white noise.
The loss \eq{\ref{eq:loss}} is calculated afterward.

\subsubsection{Inference Phase}
In the inference phase, the model $\errornetsolo$ generates a reconstruction $\mtsreconst$ given the input instance $\mtsstep{0} \in \testset$.
The inference procedure is shown in \alg{\ref{alg:naiveinference}}.
It first corrupts $\mtsstep{0}$ into $\mtsstepreconst{\reversediffstep}$ (\step{\ref{step:naiveinfcorrupt}}):
\begin{align}
    \mtsstepreconst{\reversediffstep} =
    \alphaschedulebart{\reversediffstep}\mtsstep{0} +
    \sqrt{1-\alphaschedulebart{\reversediffstep}}\errort{\reversediffstep}, \label{eq:naive_corrupt}
\end{align}
where $\reversediffstep$ is a reverse diffusion step that can be shorter than $\forwarddiffstep$.
In general, relatively weak noise is added to form $\mtsstepreconst{\reversediffstep}$ so that it retains some of the input features~\cite{anoddpm,diffusionae}.
This can be seen as a conditional DDPM on the input instance.
Then, in each step (\step{\ref{step:naiveinfdenoise}}), $\mtsstepreconst{\reversediffstep}$ is gradually denoised to reconstruct $\mtsreconst$.
\begin{align}
    \mtsstepreconst{\diffstep-1} =
    \frac{1}{\sqrt{\alphaschedulet{\diffstep}}}
    \big(
    \mtsstepreconst{\diffstep} - 
    \frac{\varschedulet{\diffstep}}{\sqrt{1-\alphaschedulebart{\diffstep}}} \errornet{\mtsstepreconst{\diffstep} , \diffstep}
    \big) +
    \varscheduletildet{\diffstep}\errort{\diffstep}. \label{eq:naive_inference}
\end{align}

\begin{figure*}[t]
    \centering
    \includegraphics[width=1\linewidth]{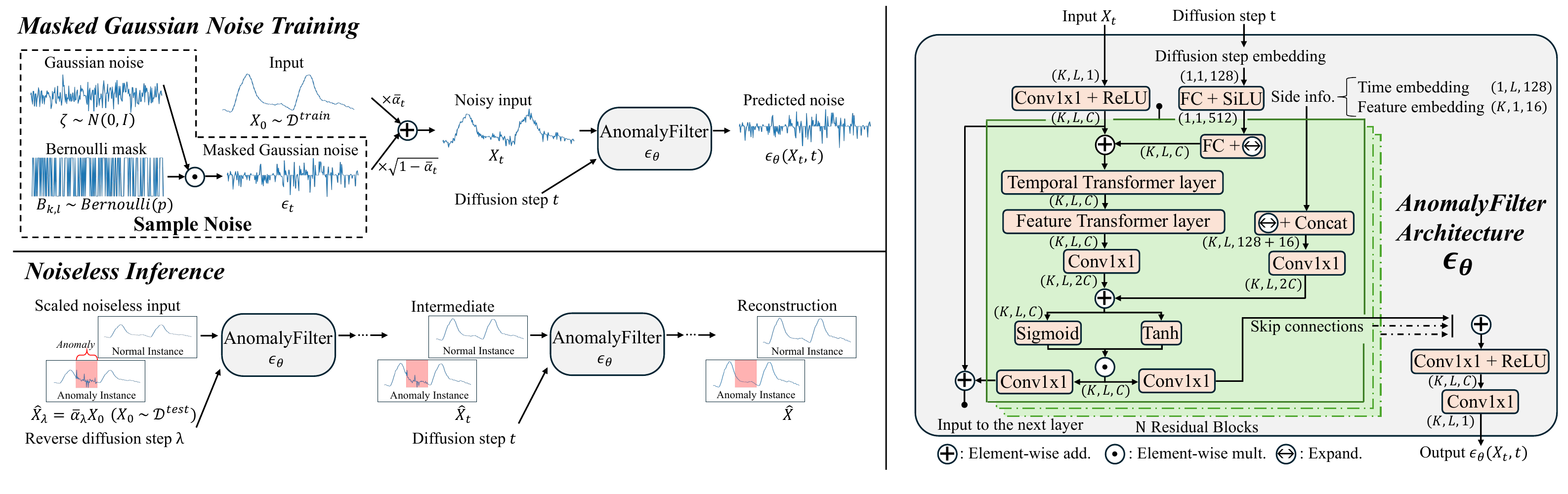} \\
    \caption{
        An overview of \method framework.
        The model is trained to predict masked Gaussian noise, with the aim of learning to denoise the anomaly parts (i.e., nonmasked parts) and retaining the normal parts (i.e., masked parts).
        Noiseless inference operates without noise at the beginning and during the denoising steps.
        Thus, when given a normal instance, it outputs an instance with no change.
        When given an anomaly instance, only anomalous parts are removed.
        \method captures temporal and inter-variable dependencies by temporal and feature transformer layers and predicts noise added to the input.
        }
    \label{fig:overview}
\end{figure*}

\section{Proposed \method}
    \label{040model}
    In this section, we propose \method, a novel diffusion-based method designed for TSAD, which works as a selective filter that removes only the anomalous parts while retaining the normal parts.
We begin by providing an overview of \method and its architecture.
Subsequently, we look at how to achieve such a filter.

\subsubsection{Overview}
An overview of the \method framework is in \fig{\ref{fig:overview}} (left).
It consists of \textit{masked Gaussian noise training} and \textit{noiseless inference}.
Masked Gaussian noise makes the model learn to allow normal parts to pass through selectively by masking the Gaussian noise.
To address the reconstruction quality issue, we propose noiseless inference, which eliminates the addition of noise during the inference.
This aligns with the learned masked Gaussian noise and achieves the accurate reconstruction of normal parts.

\subsubsection{Architecture}
The architecture of \method is shown in \fig{\ref{fig:overview}} (right) and is based on CSDI~\cite{csdi}, which is a modification of DiffWave~\cite{diffwave} employing transformer layers.
It is composed of multiple residual layers with a latent dimension $\latentdim$.
Diffusion step $\diffstep$ is embedded as a 128-dimensional encoding vector with positional encoding~\cite{transformer} and input into the model together with data $\mtsstep{\diffstep}$.
To incorporate temporal and inter-variable dependencies beyond $\timestep$ timesteps and $\ndim$ features, we employ a temporal and feature transformer layer. 
The positional encoding of time and feature is added afterward.
As is typical for diffusion models, the output dimensions match the input dimensions.

\subsection{Masked Gaussian Noise} \label{sec:maskedgaussian}
Our objective is to build a filter with two functionalities: (1) removing the anomalous parts and (2) allowing the normal parts to pass through.
We design masked Gaussian noise for this purpose.
Specifically, we represent anomalous parts with Gaussian noise, as in typical DDPM.
To enable the model to allow normal parts to pass through, we randomly mask the noise.
The stochastic selection of the noise to mask enables the filter to selectively denoise anomalous parts.
Formally, our masked Gaussian noise $\errort{\diffstep}$ is defined as follows:
\begin{align}
    \errort{\diffstep} = \mask\hadamard\gaussnoise, \label{eq:maskednoise}
\end{align}
where $\hadamard$ is element-wise multiplication,
$\gaussnoise \sim \gauss{\mathbf{0}, \idmat}$ is Gaussian noise,
and $\mask$ is a mask 
with elements drawn from a Bernoulli distribution,
$\mask_{k,l} \sim Bernoulli(\bernoulliparam)$.
The hyperparameter $\bernoulliparam$ controls the mask ratio; the larger the value, the larger the ratio of Gaussian noise.

\begin{algorithm}[t]
    \caption{Masked Gaussian Noise Training}
    \label{alg:maskedtraining}
    \begin{algorithmic}[1]
        \STATE {\bf Input:}
        Training set $\trainingset$,
        Forward diffusion step $\forwarddiffstep$, \\
        \textcolor{cyan}{Mask ratio $\bernoulliparam$};
        \REPEAT
             \STATE Sample data
            $\mtsstep{0} \sim \trainingset$;
            \STATE Sample diffusion step
            $\diffstep \sim Uniform(\{ 1, \dots, \forwarddiffstep \})$;
            \STATE \textcolor{cyan}{Sample noise
            $\errort{\diffstep} = \mask\hadamard\gaussnoise$ , \\
            where $\gaussnoise \sim \gauss{\mathbf{0}, \idmat}, \mask_{k,l} \sim Bernoulli(\bernoulliparam)$}; \label{step:masktrainsamplenoise}
            \STATE Corrupt 
            $\mtsstep{\diffstep} = 
            \alphaschedulebart{\diffstep}\mtsstep{0} + 
            \sqrt{1-\alphaschedulebart{\diffstep}}\errort{\diffstep}$;
            \STATE Calculate
            $\lossfunc =
            \lVert \errort{\diffstep} - \errornet{\mtsstep{\diffstep} ,\diffstep} \rVert$;
        \UNTIL{Converged;}
    \end{algorithmic}
    \normalsize
\end{algorithm}

\subsection{Masked Gaussian Noise Training} \label{sec:training}
The simplified DDPM loss function \eq{\ref{eq:loss}} can be interpreted as training the model to predict the noise added to the data.
Similarly, we train our model $\errornetsolo$ to predict the masked Gaussian noise $\errort{\diffstep}$.

\subsubsection{Algorithm}
The training procedure for \method is summarized in \alg{\ref{alg:maskedtraining}}, where the differences from the conventional algorithm (\alg{\ref{alg:training}}) are highlighted in \textcolor{cyan}{cyan}.
We follow the conventional training procedure for the diffusion model except for the noise $\errort{\diffstep}$ (\step{\ref{step:masktrainsamplenoise}}).
Specifically, given data $\mtsstep{0}$, we sample noisy data $\mtsstep{\diffstep}$ and train $\errornetsolo$ by minimizing the loss function in \eq{\ref{eq:loss}}.

\subsubsection{Objective Function}
Although we use a loss function similar to that of DDPM, with the introduction of masked Gaussian noise, our loss function can be divided into two parts.
\small
\begin{align}
    \lossfunc 
    &= \expectation{\lVert \errort{\diffstep} - \errornet{\mtsstep{\diffstep} , \diffstep} \rVert} \\ \nonumber
    &= \underbrace{ \expectation{\mask\hadamard \lVert \gaussnoise - \errornet{\mtsstep{\diffstep} ,\diffstep} \rVert} }_{\text{Noisy part}}
    + \underbrace{ \expectation{(\bm{1}-\mask)\hadamard \lVert \bm{0} - \errornet{\mtsstep{\diffstep},\diffstep} \rVert} }_{\text{Noiseless part}} \\ \nonumber
    &= \lossfuncnonmask + \lossfuncmask .
\end{align}
\normalsize
The derivation is provided in the Appendix.
Thus, this loss function enables the model to learn a filter with two functionalities: (1) the removal of anomalous parts ($\lossfuncnonmask$) and (2) the retention of normal parts ($\lossfuncmask$).

\begin{algorithm}[t]
    \caption{Noiseless Inference}
    \label{alg:noiselessinference}
    \begin{algorithmic}[1]
        \STATE {\bf Input:}
        Data $\mtsstep{0} \in \testset$,
        Reverse diffusion step $\reversediffstep$,
        Trained denoising function $\errornetsolo$;
        \STATE {\bf Output:}
        Reconstruction $\mtsstepreconst{0}$ (i.e., $\mtsreconst$);
        \STATE \textcolor{cyan}{Scale $\mtsstepreconst{\reversediffstep} = \alphaschedulebart{\reversediffstep}\mtsstep{0}$} ; \label{step:noiselessscale}
        \FOR{$\diffstep=\reversediffstep:1$}
            \STATE \textcolor{cyan}{Denoise
            $\mtsstepreconst{\diffstep-1} =
            \frac{1}{\sqrt{\alphaschedulet{\diffstep}}}
            \big(
            \mtsstepreconst{\diffstep} - 
            \frac{\varschedulet{\diffstep}}{\sqrt{1-\alphaschedulebart{\diffstep}}} \errornet{\mtsstepreconst{\diffstep} , \diffstep}
            \big)$}; \label{step:noiselessdenoise}
        \ENDFOR
        \STATE \textbf{return} $\mtsstepreconst{0}$;
    \end{algorithmic}
    \normalsize
\end{algorithm}
\subsection{Noiseless Inference} \label{sec:noiseless}

To achieve low reconstruction errors for normal parts, we introduce noiseless inference, which inputs the data as they are and does not add noise at each denoising step.
\alg{\ref{alg:noiselessinference}} is our proposed noiseless inference.
Instead of inputting corrupted data
$\mtsstepreconst{\reversediffstep} = \alphaschedulebart{\reversediffstep}\mtsstep{0} + \sqrt{1-\alphaschedulebart{\reversediffstep}}\errort{\reversediffstep}$
into the model,
we input the following scaled noiseless data (\step{\ref{step:noiselessscale}}):
\begin{align}
    \mtsstepreconst{\reversediffstep} = \alphaschedulebart{\reversediffstep}\mtsstep{0}. \label{eq:noiselesscorruption}
\end{align}
Then, at each denoising step, $\mtsstep{\reversediffstep}$ is gradually denoised to reconstruct $\mtsreconst$. 
This step is conducted without any addition of noise (\step{\ref{step:noiselessdenoise}}):
\begin{align}
    \mtsstepreconst{\diffstep-1} =
            \frac{1}{\sqrt{\alphaschedulet{\diffstep}}}
            \big(
            \mtsstepreconst{\diffstep} - 
            \frac{\varschedulet{\diffstep}}{\sqrt{1-\alphaschedulebart{\diffstep}}} \errornet{\mtsstepreconst{\diffstep} , \diffstep}
            \big). \label{eq:noiselessdenoise}
\end{align}

Thanks to the masked Gaussian noise, the model functions as a filter.
Thus, noiseless inference presumes that the model outputs data with no change when given normal data and outputs data with the anomalous parts removed when given anomaly data.
The key advantage of noiseless inference over naive inference (\alg{\ref{alg:naiveinference}}) is that it prevents the loss of important features of the original input, which occurs by adding noise.
This allows the input to act as a condition for reconstruction, leading to an accurate reconstruction of the normal parts.

\section{Experiments}
    \label{050experiments}
    \subsection{Experiment Setup}

\subsubsection{Datasets} 
We use five datasets in our paper:
(1) UCR anomaly archive (UCR)~\cite{current},
(2) AIOps,
(3, 4) Yahoo real and Yahoo bench,
and (5) Server Machine Dataset (\smd)~\cite{omnianomaly}.
We trained and tested separately for each of the subdatasets, and the average results of five runs are presented.

\begin{table*}[t]
    \centering
    \resizebox{1.0\linewidth}{!}{
  \begin{tabular}{l|llll|lll|lll|lll|lll}
  \toprule
   & \multicolumn{4}{c|}{UCR} & \multicolumn{3}{c|}{AIOps} & \multicolumn{3}{c|}{Yahoo real} & \multicolumn{3}{c|}{Yahoo bench} & \multicolumn{3}{c}{SMD} \\
   & V-R & V-P & RF & Acc. & V-R & V-P & RF & V-R & V-P & RF & V-R & V-P & RF & V-R & V-P & RF \\
  \midrule
\isf&0.641&0.106&0.041&0.236&0.798&0.169&0.061&0.699&0.358&0.000&0.611&0.295&0.000&0.735&0.245&0.093\\
\ocsvm&0.645&0.167&0.092&0.300&0.848&0.336&0.152&0.715&0.443&0.000&0.608&0.289&0.000&0.698&0.239&0.156\\
\beatgan&$\mathbf{0.746}$&$\underline{0.205}$&$\underline{0.137}$&$\underline{0.351}$&$\underline{0.889}$&$\underline{0.487}$&$\mathbf{0.239}$&0.751&0.503&0.060&0.720&$\underline{0.373}$&0.004&0.819&0.356&0.232\\
\lstmvae&$\underline{0.739}$&0.141&0.113&0.217&0.864&0.359&0.185&$\underline{0.899}$&$\underline{0.612}$&0.076&0.709&0.329&0.003&0.711&0.223&0.190\\
\usad&0.674&0.156&0.129&0.272&$\underline{0.889}$&0.445&0.227&0.830&0.571&0.075&0.612&0.305&0.000&0.508&0.145&0.088\\
\at&0.704&0.132&0.103&0.202&0.881&0.347&0.157&0.872&0.557&0.051&0.710&0.291&0.005&0.861&$\underline{0.420}$&0.252\\
\tranad&0.579&0.072&0.059&0.157&0.762&0.230&0.160&0.774&0.496&0.061&0.589&0.270&0.000&0.636&0.194&0.144\\
\dada&0.585&0.028&0.015&0.036&0.839&0.255&0.161&0.860&0.548&0.095&0.745&0.309&$\underline{0.007}$&$\underline{0.875}$&0.419&$\underline{0.262}$\\
\diffad&0.605&0.063&0.039&0.095&0.859&0.395&0.191&0.844&0.520&0.077&$\underline{0.768}$&0.346&0.004&0.606&0.169&0.157\\
\imdiffusion&0.630&0.047&0.028&0.060&0.866&0.358&0.126&0.860&0.552&0.062&0.725&0.300&0.002&0.850&0.341&0.196\\
\dddr&0.580&0.071&0.060&0.157&0.839&0.396&0.213&0.839&0.560&$\underline{0.099}$&0.668&0.315&0.000&0.828&0.330&0.235\\
\textbf{\method}&0.726&$\mathbf{0.238}$&$\mathbf{0.204}$&$\mathbf{0.399}$&$\mathbf{0.932}$&$\mathbf{0.597}$&$\underline{0.230}$&$\mathbf{0.900}$&$\mathbf{0.625}$&$\mathbf{0.109}$&$\mathbf{0.793}$&$\mathbf{0.416}$&$\mathbf{0.011}$&$\mathbf{0.891}$&$\mathbf{0.477}$&$\mathbf{0.273}$\\
\bottomrule
\end{tabular}
}
    \caption{
    Anomaly detection performance.
    V-R and V-P are VUS-ROC and VUS-PR~\cite{vus}.
    RF is Range F-score~\cite{rf}.
    Best results are in \textbf{bold}, and second-best results are \underline{underlined} (higher is better).
    }
    \label{table:mainresult}
\end{table*}

\begin{table*}[ht]
    \centering
    \resizebox{1.0\linewidth}{!}{
  \begin{tabular}{l|llll|lll|lll|lll|lll}
  \toprule
   & \multicolumn{4}{c|}{UCR} & \multicolumn{3}{c|}{AIOps} & \multicolumn{3}{c|}{Yahoo real} & \multicolumn{3}{c|}{Yahoo bench} & \multicolumn{3}{c}{SMD} \\
   & V-R & V-P & RF & Acc & V-R & V-P & RF & V-R & V-P & RF & V-R & V-P & RF & V-R & V-P & RF \\
  \midrule
\dae&$\mathbf{0.760}$&$\mathbf{0.252}$&$\underline{0.190}$&$\underline{0.386}$&0.900&0.501&$\mathbf{0.266}$&0.888&0.614&$\underline{0.083}$&0.694&0.310&0.003&0.866&0.435&0.238\\
\dae w/ M&0.685&0.149&0.125&0.266&0.890&0.495&0.243&0.894&0.606&0.080&0.688&0.286&0.001&$\underline{0.878}$&0.452&0.257\\
\dae w/ N&$\underline{0.758}$&0.234&0.173&0.382&0.884&0.448&$\underline{0.257}$&0.887&0.593&0.073&0.687&0.284&0.001&0.875&$\underline{0.456}$&0.255\\
\dae w/ M\&N&0.693&0.172&0.144&0.291&$\underline{0.915}$&$\underline{0.519}$&0.242&$\mathbf{0.905}$&$\underline{0.624}$&0.081&0.696&$\underline{0.316}$&$\underline{0.004}$&0.876&0.439&0.265\\
\ddpm&0.660&0.071&0.054&0.110&0.761&0.183&0.100&0.865&0.550&0.064&0.714&0.280&0.001&0.786&0.337&$\underline{0.267}$\\
\ddpm w/ M&0.660&0.062&0.046&0.097&0.713&0.162&0.079&0.843&0.508&0.074&0.717&0.268&0.000&0.751&0.286&0.228\\
\ddpm w/ N&0.709&0.165&0.127&0.272&0.827&0.323&0.204&0.896&0.613&$\underline{0.083}$&$\underline{0.727}$&0.313&0.000&0.812&0.334&0.233\\
\midrule
\hspace{-.8em}\begin{tabular}{l}\textbf{\method}\\(\ddpm w/ M\&N)\end{tabular}&0.726&$\underline{0.238}$&$\mathbf{0.204}$&$\mathbf{0.399}$&$\mathbf{0.932}$&$\mathbf{0.597}$&0.230&$\underline{0.900}$&$\mathbf{0.625}$&$\mathbf{0.109}$&$\mathbf{0.793}$&$\mathbf{0.416}$&$\mathbf{0.011}$&$\mathbf{0.891}$&$\mathbf{0.477}$&$\mathbf{0.273}$\\
  \bottomrule
\end{tabular}
}
    \caption{
    Ablation study with different components.
    M and N refer to Bernoulli mask and Noiseless inference, respectively.
    }
    \label{table:mainablation}
\end{table*}
\subsubsection{Baselines} 
We compare our method with the following 13 baselines in four categories, including classical methods and state-of-the-art reconstruction-based methods, as follows:
(1) Classical machine-learning methods, \isf~\cite{isolationforest} and \ocsvm~\cite{ocsvm},
(2) GAN, VAE, and AE-based methods, \beatgan~\cite{beatgan}, \lstmvae~\cite{lstmvae}, and \usad~\cite{usad},
(3) Transformer-based methods, \at\cite{anomalytransformer}, \tranad~\cite{tranad}, and \dada~\cite{dada},
and (4) Denoising-based methods, \diffad~\cite{diffad}, \imdiffusion~\cite{imdiffusion}, \dddr~\cite{dddr}, \dae~\cite{dae} and \ddpm~\cite{ddpm}.
We position \dae and DDPM as part of our ablation study, since they share the same architecture as \method.

\subsubsection{Evaluation Metrics} 
It is claimed that the traditional metric for TSAD, namely the F1 score calculated with a protocol named point adjustment (PA), overestimates the model performance~\cite{towards,quovadisTAD}.
Thus, we adopt various threshold-independent measures along with conventional metrics (AUC-ROC and F1 score) for a comprehensive comparison.
Specifically, we use the following five range-based measures, which are designed to provide a robust and suitable assessment for TSAD.
Range F-score~\cite{rf}, Range-AUC-ROC, Range-AUC-PR, volume under the surface (VUS)-ROC, and VUS-PR~\cite{vus}.
In particular, VUS-PR is identified as providing the most reliable and accurate measurement~\cite{elephant}. 
We set the buffer region for Range-AUC-ROC/PR at half of the window size.
For UCR, we also employ accuracy as a metric, where we consider that the model detects the anomaly if the highest anomaly score is located inside the given anomaly range, since it only has a single anomaly in a sequence.

\subsection{Anomaly Detection Performance} \label{sec:results}

\tabl{\ref{table:mainresult}} presents the performance of \method and the baseline methods for five datasets.
While the second-best method varies depending on the dataset and metric, the proposed \method achieves the highest detection accuracy across most of the datasets,
demonstrating an average improvement of $4.1\%$ in VOC-ROC, $14.4\%$ in VOC-PR, and $29.0\%$ in Range F-score compared with the second-best method.
Overall, reconstruction-based methods exceed classic machine-learning methods.
\beatgan struggles with Yahoo, which has limited training data, and SMD, which is multivariate data.
In contrast, \at performs well on SMD, but tends to overfit on simple patterns, as shown in \fig{\ref{fig:comparison}}, and underperforms on UCR.
These results highlight the difficulty of appropriately balancing the information bottleneck across diverse datasets.
\dada is a zero-shot anomaly detection method that identifies relative anomalies within a window.
However, when most values in a window are anomalous (as in UCR), it fails to detect them accurately, resulting in lower performance.
Diffusion-based methods, \diffad and \imdiffusion, face challenges in UCR, which contains many pattern-wise anomalies, reflecting limitations in the reconstruction quality.
These results demonstrate the effectiveness of \method.

\begin{table*}[ht]
    \centering
    \resizebox{1.0\linewidth}{!}{
  \begin{tabular}{l|rrr|rrr|rrr|rrr|rrr}
  \toprule
   & \multicolumn{3}{c|}{UCR} & \multicolumn{3}{c|}{AIOps} & \multicolumn{3}{c|}{Yahoo real} & \multicolumn{3}{c|}{Yahoo bench} & \multicolumn{3}{c}{SMD} \\
   & $MSE_{a}$ & $MSE_{n}$ & $\frac{MSE_a}{MSE_n}$ & $MSE_{a}$ & $MSE_{n}$ & $\frac{MSE_a}{MSE_n}$ & $MSE_{a}$ & $MSE_{n}$ & $\frac{MSE_a}{MSE_n}$ & $MSE_{a}$ & $MSE_{n}$ & $\frac{MSE_a}{MSE_n}$ & $MSE_{a}$ & $MSE_{n}$ & $\frac{MSE_a}{MSE_n}$ \\
  \midrule
\beatgan&0.0153&0.0030&5.1&0.0251&0.0005&50.6&9.6410&0.3539&27.2&0.0660&0.0296&2.2&0.0257&0.0053&4.8\\
\lstmvae&0.0171&0.0030&5.7&0.0361&0.0014&25.5&10.6639&0.3497&30.5&$\underline{0.1656}$&0.1131&1.5&$\underline{0.0387}$&0.0130&3.0\\
\usad&0.0394&0.0152&2.6&$\underline{0.0715}$&0.0068&10.5&$\underline{12.6297}$&0.7915&16.0&0.1593&0.1265&1.3&0.0036&0.0015&2.3\\
\at&0.0039&$\mathbf{0.0004}$&$\underline{10.3}$&0.0091&$\underline{0.0002}$&50.7&9.6449&0.2966&32.5&0.0266&0.0133&2.0&0.0239&0.0042&5.6\\
\tranad&0.0447&0.0328&1.4&$\mathbf{0.0892}$&0.0187&4.8&$\mathbf{12.6501}$&0.8379&15.1&$\mathbf{0.2679}$&0.2464&1.1&0.0183&0.0032&5.8\\
\dada&0.0090&0.0050&1.8&0.0072&0.0003&23.0&2.2592&$\underline{0.0204}$&$\mathbf{110.8}$&0.0195&$\mathbf{0.0059}$&$\underline{3.3}$&0.0029&$\mathbf{0.0006}$&4.8\\
\diffad&0.0373&0.0244&1.5&0.0436&0.0051&8.5&10.4784&0.2780&37.70&0.0844&0.0362&2.3&$\mathbf{0.0490}$&0.0314&1.6\\
\imdiffusion&$\underline{0.0532}$&0.0273&1.9&0.0501&0.0110&4.6&10.0372&0.2937&34.2&0.1038&0.0528&2.0&0.0150&0.0031&4.8\\
\dddr&$\mathbf{0.0540}$&0.0440&1.2&0.0476&0.0094&5.0&10.0916&0.1463&$\underline{69.0}$&0.0636&0.0252&2.5&0.0142&0.0028&5.0\\
\textbf{\method}&0.0108&$\underline{0.0008}$&$\mathbf{13.9}$&0.0344&$\mathbf{0.0001}$&$\mathbf{257.3}$&0.3551&$\mathbf{0.0075}$&47.2&0.0276&$\underline{0.0070}$&$\mathbf{4.0}$&0.0093&$\underline{0.0010}$&$\mathbf{9.5}$\\
\midrule
\dae&0.0103&0.0016&6.7&0.0374&0.0006&$\underline{66.7}$&10.3125&0.3423&30.1&0.0642&0.0382&1.7&0.0037&$\mathbf{0.0006}$&$\underline{6.0}$\\
\ddpm&0.0357&0.0170&2.1&0.0415&0.0053&7.8&0.5271&0.0716&7.4&0.0917&0.0472&1.9&0.0155&0.0053&3.0\\
DDPM w/ N&0.0261&0.0080&3.3&0.0392&0.0026&15.1&0.4989&0.0401&12.4&0.0757&0.0311&2.4&0.0139&0.0046&3.0\\
DDPM w/ M&0.0339&0.0201&1.7&0.0472&0.0111&4.3&0.4564&0.0690&6.6&0.0837&0.0511&1.6&0.0250&0.0169&1.5\\
\bottomrule
\end{tabular}
}
    \caption{
    Reconstruction quality performance.
    $MSE_a$ and $MSE_n$ are the MSE values of anomalous and normal parts during the inference.
    Not necessary, but $MSE_a$ should be higher, and $MSE_n$ should be lower.
    The higher the ratio $\frac{MSE_a}{MSE_n}$ is, the easier it is to detect anomalies.
    }
    \label{table:mseratio}
\end{table*}

\subsection{Ablation Studies} \label{sec:ablationstudy}
We assess the contribution of each component.
Qualitative analysis, hyperparameter studies, and case studies can be found in the Appendix.

\subsubsection{Effect of key components}
\tabl{\ref{table:mainablation}} shows quantitative results across datasets w/o each component with DAE and DDPM as a base method.
We can see that the noiseless inference enhances the performance of the vanilla DDPM. 
While the Bernoulli mask alone does not independently improve the accuracy of DDPM, its combination with noiseless, namely \method, leads to substantial performance gains.
This synergy underscores the importance of jointly employing these two simple components in improving anomaly detection capabilities.
The failure of the mask to work effectively on its own is because, in naive inference, noise is added at each denoising step, resulting in a mismatch with what the model has learned.
Interestingly, when DAE is used as the base method, neither the Bernoulli mask nor noiseless inference contributes to performance improvement.
This is probably because the direct prediction framework of DAE makes it challenging to learn a complex filter capable of both retaining and denoising, while DDPM has demonstrated remarkable capabilities in learning the filter.

\begin{figure}[t]
    \centering
    \begin{tabular}{cc} 
        \hspace{-1em}
        \includegraphics[height=9em]{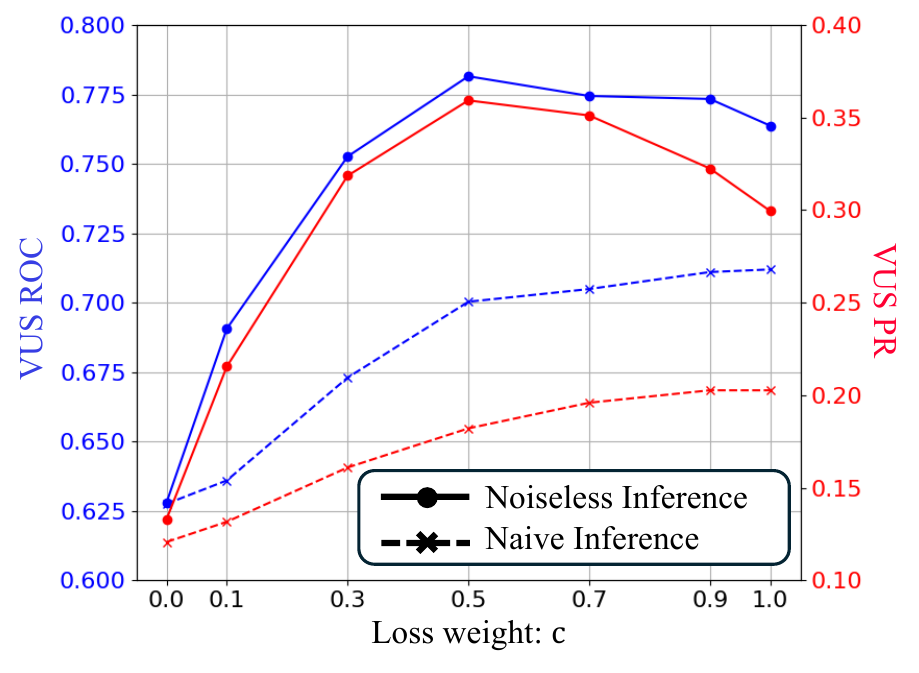} & 
        \hspace{-1em}
        \includegraphics[height=9em]{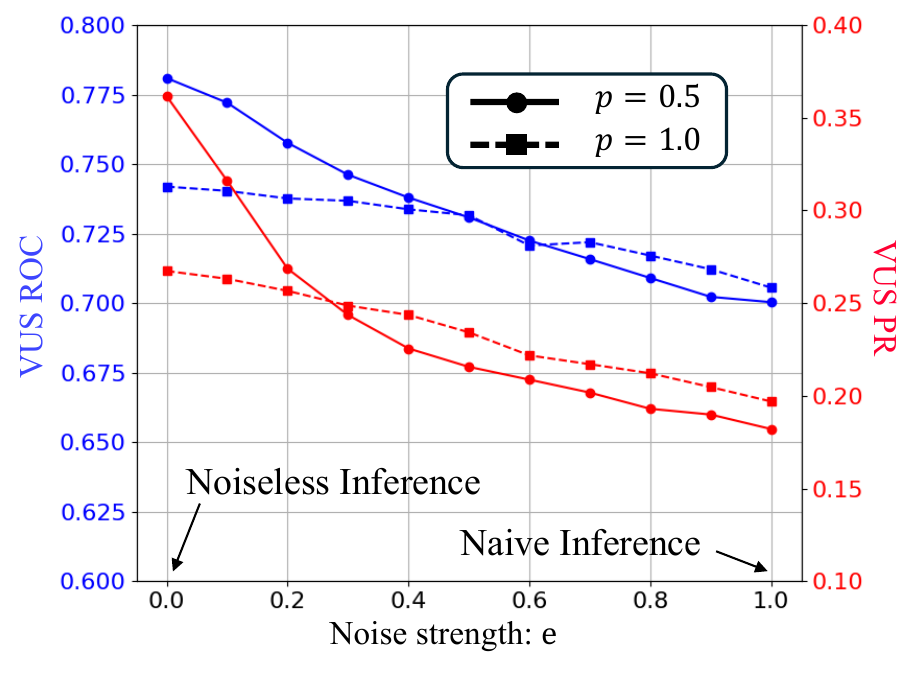} \\ 
        (a) Loss ratio $c$ & (b) Noise strength $\noiseweight$ \\ 
    \end{tabular}
    \caption{
    Effect of loss weights and noise strength.
    }
    \label{fig:ablation}
\end{figure}

\subsubsection{Effect of loss weights}
We investigate the effect of loss weights.
During our experiments, we give the same importance to each regularizer, i.e., the coefficients $\lossweight$ in 
$\lossfunc = \lossweight\lossfuncnonmask + (1-\lossweight)\lossfuncmask$
are set at 0.5.
Here, we vary $\lossweight$ from $\{0.1, \dots 1.0\}$, using $\bernoulliparam=0.5$ for training, where $\lossweight=1.0$ focused solely on denoising and $\lossweight=0.0$ focused solely on retention.
The results, shown in \fig{\ref{fig:ablation}} (a), indicate that $\lossweight=0.5$ achieved the best performance, aligning with $\bernoulliparam=0.5$ and suggesting that explicit weighting is unnecessary.
However, under naive inference, $\lossweight=1.0$ performed best, demonstrating that denoising is critical when inference includes noise, whereas both processes are significant when inference is noise-free.

\subsubsection{Effect of noise strength during inference}
We evaluate the effect of noiseless inference by modifying the noise added during naive inference (\alg{\ref{alg:naiveinference}}, Lines 3 and 6) from $\errort{t}$ to $\noiseweight\errort{t}$ and gradually reducing $\noiseweight$ from $\{1.0, \dots, 0.0\}$, even achieving noiseless inference.
This is equivalent to sampling Gaussian noise $\gaussnoise$ from $\gauss{\mathbf{0}, \noiseweight^{2}\idmat}$.
\fig{\ref{fig:ablation}} (b) shows the results for models trained with masked Gaussian noise ($\bernoulliparam=0.5$) and Gaussian noise ($\bernoulliparam=1.0$).
In both cases, the performance improved as $\noiseweight$ decreased, with the best results achieved at $\noiseweight=0.0$ (i.e., noiseless inference), indicating that noise is unnecessary for inference in TSAD.
The model trained on Gaussian noise outperformed masked Gaussian noise at around $\noiseweight \geq 0.4$, suggesting that models trained with Gaussian noise are more suitable when inference includes higher levels of noise, aligning with their training paradigm.

\subsection{Reconstruction Quality} \label{sec:reconst}
We show the MSE of anomalous and normal parts and their ratio in \tabl{\ref{table:mseratio}}.
Compared to baselines and ablation settings, \method achieves a significantly lower $MSE_{n}$, resulting in a larger ratio $\frac{MSE_a}{MSE_n}$.
This demonstrates that \method is well-suited for effective anomaly detection.
Although \at achieves a very small $MSE_{n}$ on UCR and AIOps, its $MSE_{a}$ is similarly small, resulting in a mediocre ratio.
As illustrated in \fig{\ref{fig:comparison}}, this causes false positives by even slight misalignments in normal regions.
\dada, a zero-shot anomaly detector, shows favorable ratios on datasets like Yahoo, where the number of training samples is limited and point anomalies are dominant, but struggles in UCR, which contains longer anomalous parts than the window size.
Diffusion-based methods such as \diffad, \imdiffusion, and DDPM exhibit high values in both $MSE_{a}$ and $MSE_{n}$, indicating the limitation of diffusion models.
DDPM w/ Mask yields results similar to the vanilla DDPM, suggesting that the Bernoulli mask by itself is not sufficient for improving detection ability.
In contrast, DDPM w/ Noiseless setting shows a notable reduction in $MSE_{n}$, leading to improved ratio.
These findings are consistent with the anomaly detection performance reported in \tabl{\ref{table:mainablation}} and further support the effectiveness of the proposed approach, particularly the combined contribution of the Bernoulli mask and noiseless inference.

\section{Conclusion}
    \label{070conclusions}
    We proposed a novel diffusion-based method called \method for unsupervised TSAD.
To resolve the low reconstruction quality of existing reconstruction-based methods, we built a selective filter that accurately denoises anomalous parts while retaining normal parts.
Specifically, we investigated the utility of noise in diffusion models and devised masked Gaussian noise and noiseless inference.
Extensive experiments on five datasets demonstrated the effectiveness of the method.
Ablation studies confirmed the contribution of our design to anomaly detection in a quantitative and qualitative manner and demonstrated that the synergy of the proposed two simple components strongly enhances vanilla DDPM.
Overall, we believe that our work paves the way for a new diffusion-based approach towards TSAD.

\section{Acknowledgments}
The authors would like to thank the anonymous referees for their valuable comments and helpful suggestions.
This work was supported by
JST CREST JPMJCR23M3,
JST START JPMJST2553,
JST CREST JPMJCR20C6,
JST K Program JPMJKP25Y6,
JST COI-NEXT JPMJPF2009,
JST COI-NEXT JPMJPF2115,
the Future Social Value Co-Creation Project - Osaka University.

\bibliography{ref_denoising}

\clearpage

\section{Proof} \label{apdx:lossproof}
\begin{prop}
    Assume we have a noise $\errort{\diffstep} = \mask\hadamard\error^{1} + (\bm{1}-\mask)\hadamard\error^{2}$,
    where $\error^{1}$ and $\error^{2}$ are sampled from some distributions, and $\mask$ is a mask with elements drawn from a Bernoulli distribution, $\mask_{k,l} \sim Bernoulli(\bernoulliparam)$.
    Then, the loss function \eq{\ref{eq:loss}} can be separated into two parts by the mask.
\end{prop}
\begin{prf}
Each element of the mask is independent.
Hence,
\begin{align}
    \lossfunc 
    &= \expectation{\lVert \errort{\diffstep} - \errornet{\mtsstep{\diffstep} ,\diffstep} \rVert} \\ \nonumber
    &= \expectation{\lVert \mask\hadamard\errort{\diffstep} - \mask\hadamard\errornet{\mtsstep{\diffstep} ,\diffstep} \rVert} \\ \nonumber
    &\;\;\;+ \expectation{\lVert (\bm{1}-\mask)\hadamard\errort{\diffstep} - (\bm{1}-\mask)\hadamard\errornet{\mtsstep{\diffstep},\diffstep} \rVert} \\ \nonumber
    &= \expectation{\lVert \mask\hadamard \mask\hadamard \error^{1} - \mask\hadamard\errornet{\mtsstep{\diffstep} ,\diffstep} \rVert} \\ \nonumber
    &\;\;\;+ \expectation{\lVert (\bm{1}-\mask)\hadamard (\bm{1}-\mask) \hadamard \error^{2} - (\bm{1}-\mask)\hadamard\errornet{\mtsstep{\diffstep},\diffstep} \rVert}  \\ \nonumber
    &= \expectation{\lVert \mask\hadamard\error^{1} - \mask\hadamard\errornet{\mtsstep{\diffstep} ,\diffstep} \rVert} \\ \nonumber
    &\;\;\;+ \expectation{\lVert (\bm{1}-\mask) \hadamard \error^{2} - (\bm{1}-\mask)\hadamard\errornet{\mtsstep{\diffstep},\diffstep} \rVert}   \\ \nonumber
    &= \expectation{\mask\hadamard \lVert \error^{1} - \errornet{\mtsstep{\diffstep} ,\diffstep} \rVert}
    + \expectation{(\bm{1}-\mask)\hadamard \lVert \error^{2} - \errornet{\mtsstep{\diffstep},\diffstep} \rVert} .  \\ \nonumber
\end{align}
\end{prf}

Our masked Gaussian noise is a special case of the above, where $\error^{1} = \gaussnoise \sim \gauss{\mathbf{0}, \idmat}$ and $\error^{2} = \mathbf{0}$ is deterministic.

Thus, we have
\begin{align}
    \lossfunc 
    &= \expectation{\lVert \errort{\diffstep} - \errornet{\mtsstep{\diffstep} , \diffstep} \rVert} \\ \nonumber
    &= \underbrace{ \expectation{\mask\hadamard \lVert \gaussnoise - \errornet{\mtsstep{\diffstep} ,\diffstep} \rVert} }_{\text{Noisy part}}
    + \underbrace{ \expectation{(\bm{1}-\mask)\hadamard \lVert \bm{0} - \errornet{\mtsstep{\diffstep},\diffstep} \rVert} }_{\text{Noiseless part}} \\ \nonumber
    &= \lossfuncnonmask + \lossfuncmask .
\end{align}

\section{Experiment Setup}

\subsection{Datasets} \label{apdx:datasets}
\begin{table}[t]
    \centering
    \resizebox{1.0\linewidth}{!}{
    \begin{tabular}{l|l|l|l|l|l}
        \toprule
         & \begin{tabular}{c}UCR\end{tabular} & \begin{tabular}{c}AIOps\end{tabular} & \begin{tabular}{c}Yahoo\\real\end{tabular} & \begin{tabular}{c}Yahoo\\bench\end{tabular} & \begin{tabular}{c}SMD\end{tabular} \\
        \midrule
        \# of Subdatasets & 250 & 29 & 37 & 181 & 28 \\
        \# of Features & 1 & 1 & 1 & 1 & 38 \\
        Domain & Various & Cloud KPIs & Web traffic & Synthetic & Server \\
        \multicolumn{6}{c}{\textit{Train (Average per subdataset)}} \\
        Sequence Length & 30522 & 103588 & 860 & 1008 & 25300 \\
        \# of Anomalies & 0.00 & 42.35 & 0.84 & 5.53 & - \\
        Anomaly Ratio (\%)& 0.00 & 2.49 & 0.87 & 0.55 & - \\
        \multicolumn{6}{c}{\textit{Test (Average per subdataset)}} \\
        Seqence Length & 56173 & 100649 & 573 & 672 & 25300 \\
        \# of Anomalies & 1.00 & 50.69 & 2.49 & 3.97 & 11.68 \\
        Anomaly Ratio (\%)& 0.85 & 1.81 & 4.58 & 0.59 & 4.21 \\
        \bottomrule
    \end{tabular}
    }
    \caption{Statistics of the datasets.}
    \label{table:dataset}
\end{table}

We provide a detailed description of the datasets and their statistics (\tabl{\ref{table:dataset}}).

\begin{itemize}
    \item UCR anomaly archive (UCR)~\cite{current}:
    contains 250 diverse univariate time series from various domains.
    According to \cite{modelselection}, the datasets include nine domains,
    (1) Acceleration, (2) Air Temperature, (3) Arterial Blood Pressure (ABP), (4) Electrical Penetration Graph (EPG), (5) Electrocardiogram (ECG), (6) Gait, (7) NASA, (8) Power Demand, and (9) Respiration (RESP).
    There is only one well-checked anomaly in each sequence, which overcomes some of the flaws in previously used benchmarks.

    \item AIOps
    \footnote{\url{https://competition.aiops-challenge.com/home/competition/1484452272200032281}}:
    comprises 29 univariate time series of Key Performance Indicators (KPIs), which are monitored to provide a stable web service from prominent Internet companies.

    \item Yahoo real and bench
    \footnote{\url{https://webscope.sandbox.yahoo.com/catalog.php?datatype=s&did=70}}:
    are provided as part of the Yahoo! Webscope program.
    Yahoo real contains real production traffic to some of the Yahoo! properties, and Yahoo bench contains synthetic time series, including components such as trend and seasonality.

    \item Server Machine Dataset (\smd)~\cite{omnianomaly}:
    comprises a multivariate time series with 38 features from 28 server machines collected from a large internet company over a period of five weeks.
    Both the training and test sets contain around 50k timestamps with $5\%$ anomalous cases.
    Although the use of \smd has generated both support~\cite{dghl,modelselection} and opposition~\cite{current}, we employ this dataset because of its multivariate nature.

\end{itemize}

\subsection{Baselines} \label{apdx:baselines}
We detail the baselines below.
\begin{itemize}
    \item \isf~\cite{isolationforest}: constructs binary trees based on random space splitting.
    The nodes with shorter path lengths to the root are more likely to be anomalies. 
    \item \ocsvm~\cite{ocsvm}: uses the smallest possible hypersphere to encompass instances.
    The distance from the hypersphere to data located outside the hypersphere is used to determine whether or not the data are anomalous. 
    \item \beatgan~\cite{beatgan}: is based on a GAN framework where reconstructions produced by the generator are regularized by the discriminator instead of fixed reconstruction loss functions.
    \item \lstmvae~\cite{lstmvae}: is an LSTM-based VAE capturing complex temporal dependencies that models the data-generating process from the latent space to the observed space, which is trained using variational techniques.
    \item \usad~\cite{usad}: presents an autoencoder architecture whose adversarial-style learning is also inspired by GAN. 
    \item \at~\cite{anomalytransformer}: combines series and prior association to make anomalies distinctive. 
    \item \tranad~\cite{tranad}: is a Transformer-based model that uses attention-based sequence encoders to perform inferences with broader temporal trends, focusing on score-based self-conditioning for robust multi-modal feature extraction and adversarial training. 
    \item \dada~\cite{dada}: is a zero-shot anomaly detector designed for multi-domain time series. It features Adaptive Bottlenecks for flexible latent space compression based on data characteristics, and Dual Adversarial Decoders that explicitly distinguish normal from anomalous patterns through adversarial training. 
    \item \diffad~\cite{diffad}: is a conditional diffusion model that employs a density ratio-based strategy as a conditioner.
    Its conditioner is based on density-based selection, where only the high-density (i.e., normal) parts of the sample are provided to the model.
    \item \imdiffusion~\cite{imdiffusion}: integrates the imputation method with a grating masking strategy and conditional diffusion model.
    It masks half of a sequence and uses the unmasked portion as the conditioner to reconstruct the masked part.
    Then, it reverses the mask and reconstructs the entire sequence.
    The reconstruction is gradually generated from white noise through the denoising process.
    \item \dddr~\cite{dddr}: is a denoising model that tackles the drift via decomposition and reconstruction, overcoming the limitation of the local sliding window.
    It first extracts stable components by decomposition.
    Then, it disturbs the stable components by adding noise and learns denoising.
    The reconstruction is conducted at 1 step, similar to \dae.
    \item \dae~\cite{dae}: corrupts the sample with $\mtsreconst = \mts + \alpha\error$, where we use $\alpha = 0.1, \error \sim \gauss{\mathbf{0}, \idmat}$, and directly predicts the clean sample $\mts$.
    We use the same hyperparameters and model architecture with \method.
    \item DDPM~\cite{ddpm}: trains the model by \alg{\ref{alg:training}} and reconstructs learned samples from corrupted data by \alg{\ref{alg:naiveinference}}.
    We use the same hyperparameters and model architecture with \method.
\end{itemize}

To ensure a fair comparison of reconstruction quality, all reconstruction-based methods used in this study employed MSE as the anomaly score.
To handle sequential information, we then apply the moving average of half of the window size (50 timesteps) to the score.

\subsection{Implementation Details} \label{sec:imple}
\begin{table}[t]
    \centering
    \begin{tabular}{l|p{4cm}}
        \toprule
        Design Parameter & Value \\
        \midrule
        \multicolumn{2}{c}{\textit{Architecture}} \\
        Attention heads & 8 \\
        \# of residual blocks $N$ & 8 \\
        Latent dimension $\latentdim$ & 64\\
        \midrule
        \multicolumn{2}{c}{\textit{Diffusion model}} \\
        Forward diffusion step $\forwarddiffstep$ & 50 \\
        Reverse diffusion step $\reversediffstep$ & 50 \\
        Noise scheduling $\varschedule$ & np.linspace($\varschedule_{start}$, $\varschedule_{end}$, $\forwarddiffstep$) \\ 
        Minimum noise $\varschedule_{start}$ & 0.0001  \\
        Maximum noise $\varschedule_{end}$ & 0.01  \\
        Bernoulli parameter $\bernoulliparam$ & 0.5 \\
        \bottomrule
    \end{tabular}
    \caption{
    Key hyperparameters for \method.
    }
    \label{table:hyperparameter}
\end{table}

We train the proposed model using AdamW with a learning rate of $10^{-3}$.
The batch size is 64.
We use $10\%$ of the training samples as validation and apply early stopping for a maximum of 100 epochs.
Other hyperparameters related to the framework are listed in \tabl{\ref{table:hyperparameter}}.
Due to the out-of-memory issue, we used $N=4, \latentdim=32$ for SMD.
The baselines are implemented based on previously reported hyperparameters and are trained with early stopping using a procedure similar to \method.
The window size is 100 for all datasets.

\section{Experimental Results}

\subsection{Results of Various Metrics} \label{apdx:results}
\begin{table*}[ht]
    \centering
    \resizebox{1.0\linewidth}{!}{
  \begin{tabular}{l|llll|llll|llll|llll|llll}
  \toprule
   & \multicolumn{4}{c|}{UCR} & \multicolumn{4}{c|}{AIOps} & \multicolumn{4}{c|}{Yahoo real} & \multicolumn{4}{c|}{Yahoo bench} & \multicolumn{4}{c}{SMD} \\
   & R-R & R-P & A-R & F & R-R & R-P & A-R & F & R-R & R-P & A-R & F & R-R & R-P & A-R & F & R-R & R-P & A-R & F \\
  \midrule
\isf&0.653&0.114&0.613&0.057&0.825&0.217&0.824&0.187&0.813&0.595&0.601&0.000&0.705&0.487&0.502&0.000&0.764&0.275&0.799&0.171\\
\ocsvm&0.662&0.187&0.618&0.092&0.874&0.463&$\mathbf{0.866}$&0.220&0.841&0.722&0.633&0.000&0.710&0.516&0.523&0.000&0.727&0.274&0.801&0.254\\
\beatgan&$\mathbf{0.752}$&$\underline{0.209}$&$\mathbf{0.729}$&$\underline{0.158}$&$\underline{0.912}$&$\underline{0.594}$&0.814&0.214&0.852&0.723&0.655&0.063&0.799&$\underline{0.577}$&0.624&0.005&0.829&0.383&0.857&0.325\\
\lstmvae&$\underline{0.745}$&0.140&$\underline{0.719}$&0.117&0.881&0.419&0.830&0.197&$\mathbf{0.934}$&$\underline{0.782}$&$\mathbf{0.846}$&0.092&0.781&0.520&0.617&0.003&0.732&0.239&0.798&0.266\\
\usad&0.684&0.157&0.651&0.138&0.908&0.541&$\underline{0.862}$&0.221&0.892&0.758&0.771&0.081&0.695&0.485&0.527&0.000&0.550&0.183&0.583&0.118\\
\at&0.720&0.138&0.675&0.106&0.901&0.442&0.847&0.146&0.915&0.749&0.799&0.069&0.790&0.487&0.604&0.005&0.874&$\underline{0.457}$&$\mathbf{0.885}$&$\mathbf{0.342}$\\
\tranad&0.609&0.072&0.537&0.059&0.789&0.273&0.798&$\mathbf{0.261}$&0.844&0.666&0.712&0.068&0.678&0.449&0.502&0.000&0.693&0.259&0.604&0.159\\
\dada&0.610&0.031&0.542&0.016&0.863&0.306&0.764&0.156&0.911&0.727&0.768&$\underline{0.113}$&0.822&0.508&0.624&$\underline{0.006}$&$\underline{0.885}$&0.445&0.867&$\underline{0.326}$\\
\diffad&0.623&0.065&0.578&0.044&0.893&0.500&0.783&0.192&0.899&0.725&0.760&0.088&$\underline{0.838}$&0.560&$\underline{0.670}$&0.004&0.644&0.188&0.681&0.199\\
\imdiffusion&0.651&0.050&0.596&0.032&0.889&0.435&0.795&0.116&0.907&0.751&0.782&0.075&0.799&0.489&0.605&0.001&0.865&0.367&0.867&0.229\\
\dddr&0.610&0.074&0.534&0.061&0.871&0.492&0.804&0.199&0.902&0.767&0.764&0.109&0.765&0.507&0.543&0.000&0.846&0.368&0.833&0.321\\
\textbf{\method}&0.740&$\mathbf{0.245}$&0.697&$\mathbf{0.196}$&$\mathbf{0.943}$&$\mathbf{0.690}$&0.854&$\underline{0.227}$&$\underline{0.933}$&$\mathbf{0.813}$&$\underline{0.833}$&$\mathbf{0.128}$&$\mathbf{0.852}$&$\mathbf{0.621}$&$\mathbf{0.686}$&$\mathbf{0.013}$&$\mathbf{0.900}$&$\mathbf{0.508}$&$\underline{0.873}$&0.321\\
\bottomrule
\end{tabular}
}
    \caption{
    R-R, R-P, A-R, and F are Range-AUC-ROC, Range-AUC-PR~\cite{vus}, AUC-ROC, and F1. 
    }
    \label{table:raucresult}
\end{table*}

\begin{table*}[t]
    \centering
    \resizebox{1.0\linewidth}{!}{
  \begin{tabular}{l|llll|lll|lll|lll|lll}
  \toprule
   & \multicolumn{4}{c|}{UCR} & \multicolumn{3}{c|}{AIOps} & \multicolumn{3}{c|}{Yahoo real} & \multicolumn{3}{c|}{Yahoo bench} & \multicolumn{3}{c}{SMD} \\
   & V-R & V-P & RF & Acc. & V-R & V-P & RF & V-R & V-P & RF & V-R & V-P & RF & V-R & V-P & RF \\
  \midrule
\isf&0.262&0.201&0.125&0.425&0.103&0.125&0.077&0.220&0.193&0.000&0.195&0.176&0.000&0.145&0.189&0.104\\
\ocsvm&0.271&0.267&0.203&0.459&0.115&0.201&0.132&0.225&0.218&0.000&0.210&0.173&0.000&0.159&0.202&0.199\\
\beatgan&0.232&0.280&0.206&0.478&0.116&0.222&0.129&0.217&0.224&0.114&0.186&0.229&0.019&0.144&0.226&0.189\\
\lstmvae&0.208&0.231&0.186&0.412&0.117&0.207&0.125&0.125&0.269&0.161&0.195&0.223&0.021&0.147&0.181&0.210\\
\usad&0.235&0.257&0.233&0.446&0.123&0.235&0.145&0.196&0.267&0.148&0.235&0.214&0.000&0.201&0.135&0.123\\
\at&0.227&0.239&0.198&0.401&0.101&0.166&0.096&0.149&0.246&0.097&0.159&0.191&0.026&0.122&0.217&0.216\\
\tranad&0.215&0.165&0.177&0.364&0.149&0.177&0.136&0.219&0.272&0.139&0.221&0.183&0.008&0.228&0.172&0.123\\
\dada&0.200&0.065&0.053&0.187&0.117&0.179&0.090&0.173&0.295&0.192&0.151&0.195&0.028&0.113&0.199&0.174\\
\diffad&0.244&0.155&0.111&0.294&0.134&0.196&0.130&0.174&0.272&0.169&0.139&0.204&0.027&0.168&0.139&0.139\\
\imdiffusion&0.203&0.121&0.089&0.238&0.128&0.202&0.144&0.172&0.256&0.143&0.162&0.198&0.014&0.122&0.215&0.211\\
\dddr&0.219&0.162&0.171&0.364&0.157&0.237&0.137&0.169&0.238&0.163&0.202&0.186&0.006&0.110&0.214&0.197\\
\textbf{\method}&0.234&0.308&0.267&0.490&0.098&0.166&0.100&0.143&0.237&0.179&0.162&0.229&0.029&0.112&0.227&0.194\\
\bottomrule
\end{tabular}
}
    \caption{
    Standard deviation of VUS-ROC, VUS-PR, and Range F-score.
    }
    \label{table:mainresultstd}
\end{table*}
\begin{figure}[t]
    \centering
    \begin{tabular}{c}
        \includegraphics[width=1\linewidth]{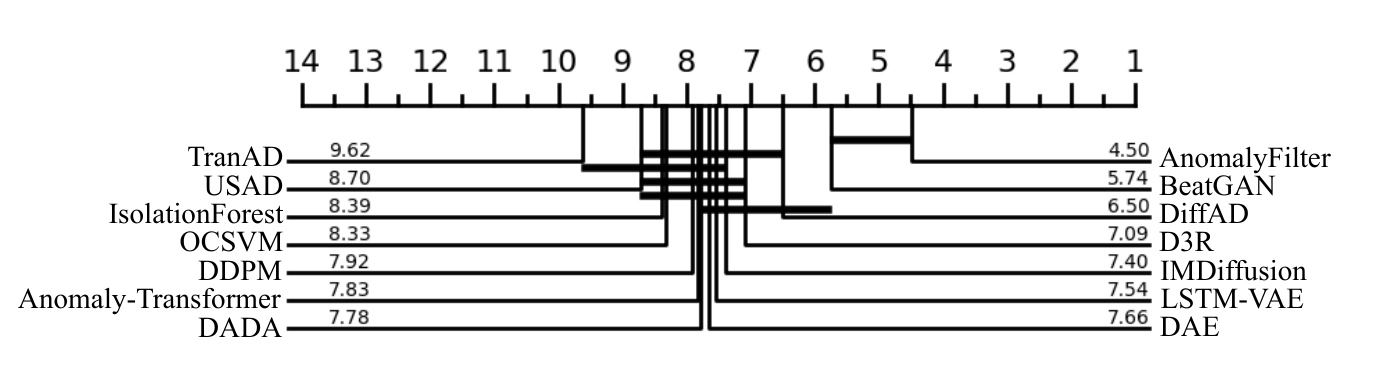}
    \end{tabular}
    \caption{
        Critical difference diagram of VUS-PR.
    }
    \label{fig:cd}
\end{figure}
We show the corresponding results of other range-based metrics, Range-AUC-ROC and Range-AUC-PR, in \tabl{\ref{table:raucresult}}, as well as standard deviation in \tabl{\ref{table:mainresultstd}}.
\fig{\ref{fig:cd}} shows the corresponding critical difference diagram of VUS-PR based on the Wilcoxon-Holm method~\cite{ismail2019deep}, where methods that are not connected by a bold line are significantly different in average rank.
Overall, we observe that \method outperforms the baselines for most datasets and metrics.

\subsection{Ablation Studies} \label{apdx:ablation}

\subsubsection{Qualitative analysis}

\begin{figure*}[t]
    \centering
    \includegraphics[width=1\linewidth]{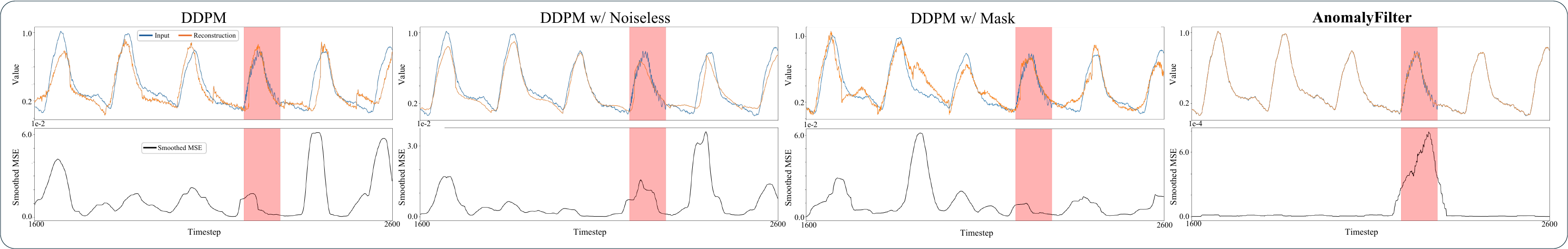} \\
    \caption{
        Reconstruction examples across DDPM variants of $\#029$ ABP on UCR. 
        }
    \label{fig:ablation_case029}
\end{figure*}
We visualize the reconstruction when adding each component to DDPM.
\fig{\ref{fig:ablation_case029}} show the results of DDPM with each component and \method of $\#029$ ABP data from UCR.
The area where it appears that the noise is added to the normal peak is the anomaly, indicated with a red background.
DDPM and DDPM w/ Mask struggle to accurately reconstruct the normal parts, leading to poor anomaly detection performance.
DDPM w/ Noiseless improves reconstruction of the normal parts compared with DDPM, but still exhibits significant deviations in certain areas.
In contrast, \method achieves notably low MSE for the normal parts, making anomalies stand out and enabling much more effective detection.

\subsection{Sensitivity Analysis} \label{sec:hyperparameter}

\begin{figure*}[t]
    \centering
    \begin{tabular}{cccc}
        \hspace{-2em}
        \begin{tabular}{c}
            \includegraphics[height=9em]{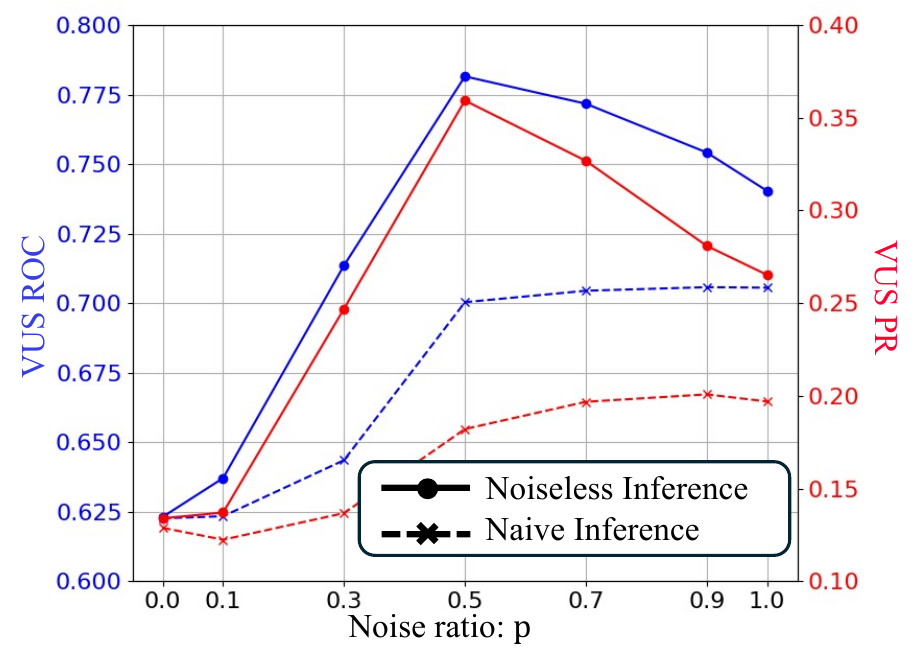} \\
            (a) Noise ratio $\bernoulliparam$
        \end{tabular} &
        \hspace{-2em}
        \begin{tabular}{c}
            \includegraphics[height=9em]{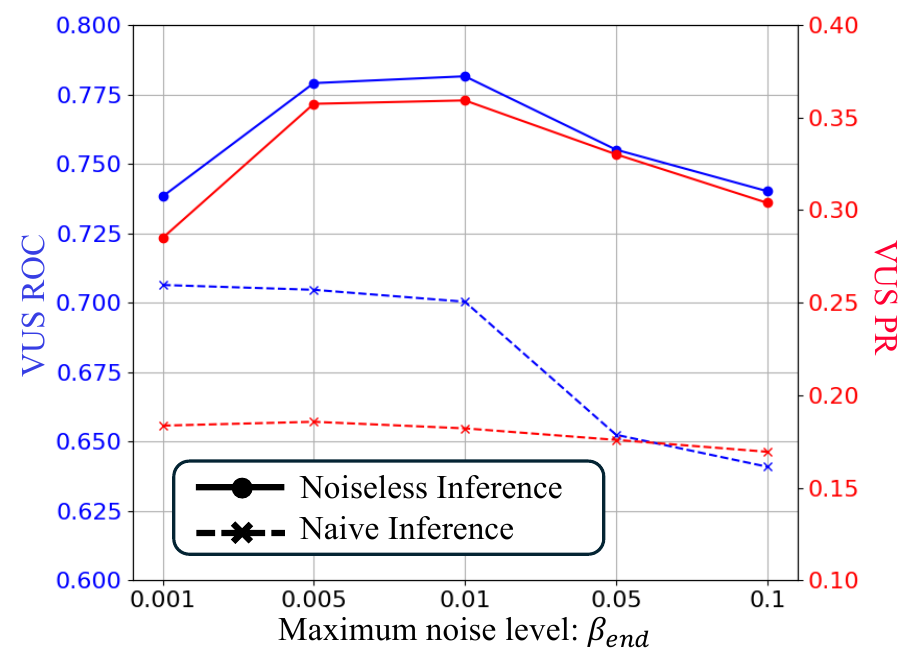} \\
            (b) Maximum noise $\varschedule_{end}$
        \end{tabular} &
        
        \hspace{-2em}
        \begin{tabular}{c}
            \includegraphics[height=9em]{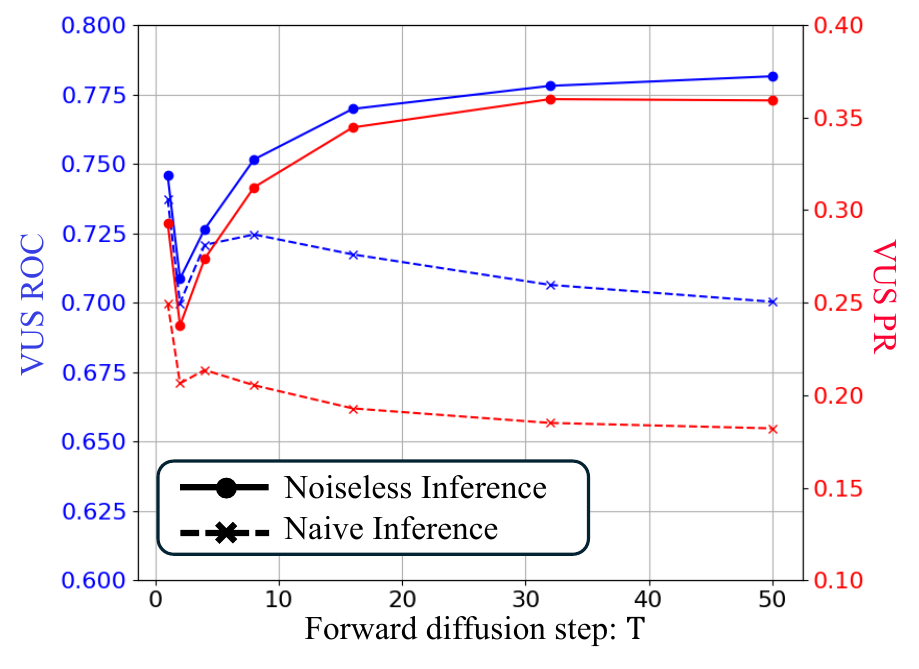} \\
            (c) Forward diffusion step $\forwarddiffstep$
        \end{tabular} &
        \hspace{-2em}
        \begin{tabular}{c}
            \includegraphics[height=9em]{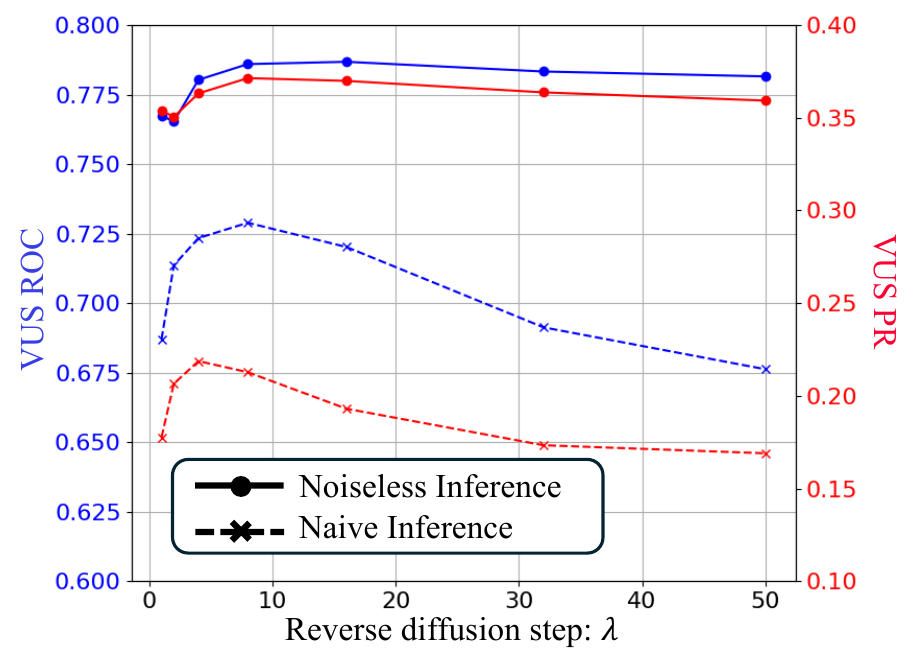} \\
            (d) Reverse diffusion step $\reversediffstep$
        \end{tabular} \\
    \end{tabular}
    \caption{
        Sensitivity analysis across four hyperparameters.
    }
    \label{fig:ablation_appendix}
\end{figure*}

We study the effect of key hyperparameters.
\fig{\ref{fig:ablation_appendix}} shows VUS-ROC and VUS-PR of the average of all datasets.

\subsubsection{Effect of ratio between noise and mask}
We investigate the effect of the Bernoulli mask by training models with masked Gaussian noise using different noise ratios $\bernoulliparam$.
\fig{\ref{fig:ablation_appendix}} (a) shows the results for both noiseless inference and naive inference. 
When $\bernoulliparam=0.0$, Gaussian noise is not added at all, and the model learns to scale the data based on the noise schedule, resulting in low accuracy with both inference procedures.
$\bernoulliparam=1.0$ indicates the addition of Gaussian noise.
For noiseless inference, the best performance was achieved at $\bernoulliparam=0.5$, demonstrating that effective learning of the filter requires balancing both the denoise and retain functionalities.
On the other hand, for naive inference, $\bernoulliparam=1.0$ yielded the best performance, indicating that reconstruction from corrupted data requires the model to be trained with noise that aligns closely with the inference conditions.

\subsubsection{Effect of noise schedule}
We use a linear variance schedule to define the noise schedule.
$\varschedule_{end}$ determines the maximum noise level, and we vary it across $\{0.001, 0.005, 0.01, 0.05, 0.1\}$.
The results in \fig{\ref{fig:ablation_appendix}} (b) indicate that an optimal level of noise is necessary for the best performance.

\subsubsection{Effect of forward diffusion step}
We vary the forward diffusion step $\forwarddiffstep$ across $\{1, 2, 4, 8, 16, 32, 50\}$, with the reverse diffusion step $\reversediffstep$ set at the same value as $\forwarddiffstep$.
Shorter steps lead to more rapid corruption.
When $\forwarddiffstep=1$, the model denoises in a single step, resembling DAE and achieving similar accuracy.
The results in \fig{\ref{fig:ablation_appendix}} (c) indicate that a $\forwarddiffstep \geq 32$ is sufficient for effective performance.
Shorter steps are preferable when using naive inference, suggesting that the denoising step should ideally operate in a noise-free setting.

\subsubsection{Effect of reverse diffusion step}
We fix the forward diffusion step $\forwarddiffstep=50$ and vary the reverse diffusion step $\reversediffstep$ across $\{1, 2, 4, 8, 16, 32, 50\}$.
As shown in \fig{\ref{fig:ablation_appendix}} (d) $\reversediffstep \geq 8$ is sufficient for effective performance.
With naive inference, the best accuracy is achieved at $\reversediffstep=8$, indicating that using a relatively small $\reversediffstep$ compared with $\forwarddiffstep$ improves performance.
However, the results also show that an overly short $\reversediffstep$ compared with $\forwarddiffstep$ hinders effective denoising, which requires careful tuning of $\reversediffstep$.

\subsection{Computational Time} \label{sec:time}
\begin{figure}[t]
    \centering
    \hspace{-2em}
    \begin{tabular}{cc} 
        \includegraphics[height=9em]{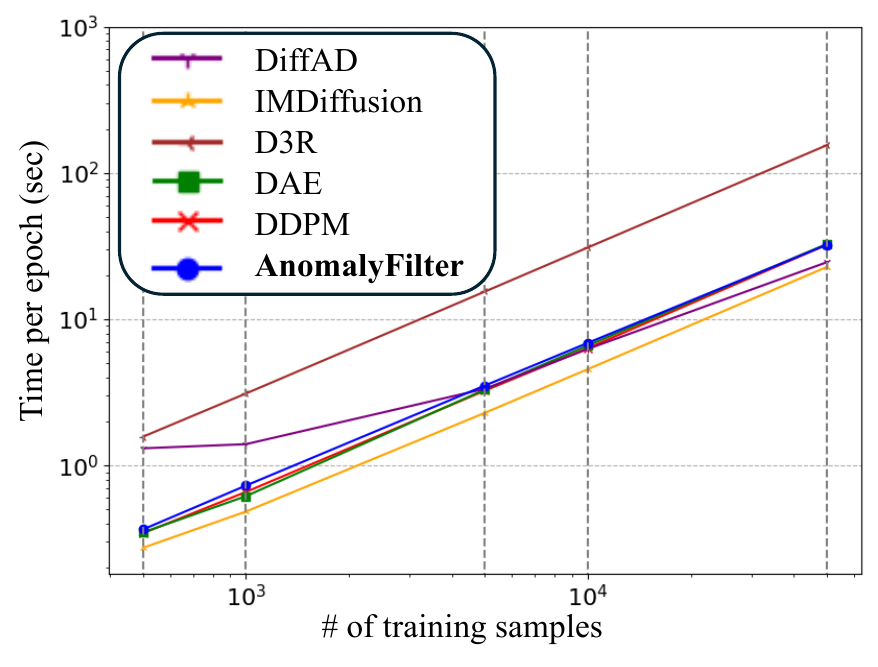} &
        \hspace{-1em}
        \includegraphics[height=9em]{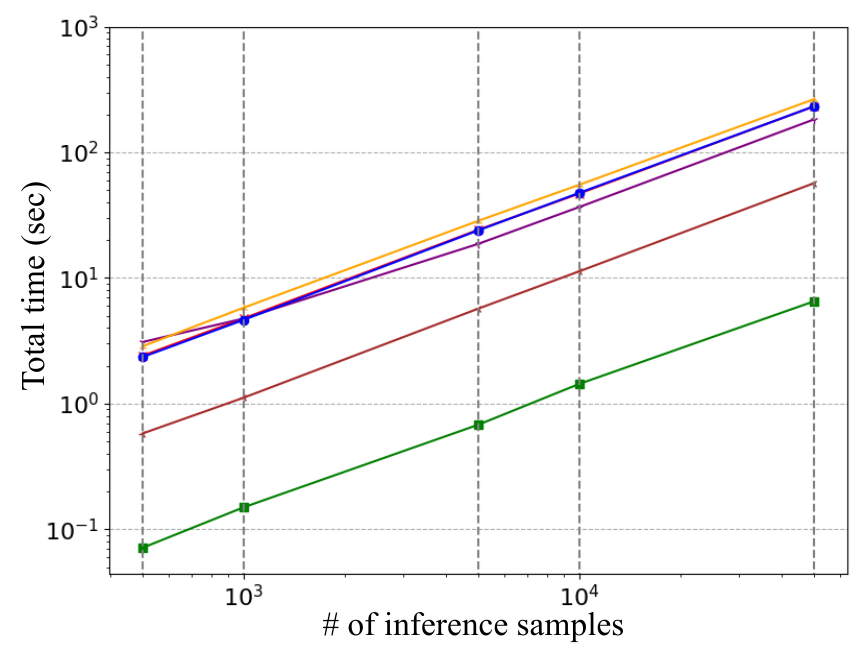} \\ 
        (a) Training & (b) Inference \\ 
    \end{tabular}
    \caption{
    Computational time for training and inference.
    }
    \label{fig:scalable}
\end{figure}

\begin{figure*}[t]
    \centering
    \includegraphics[width=1\linewidth]{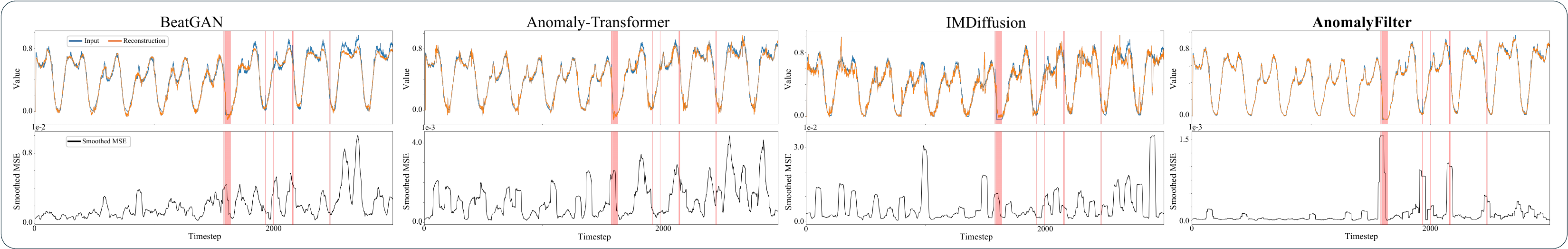} \\
    (a) AIOps KPI-301c \\
    \vspace{.5em}
    \includegraphics[width=1\linewidth]{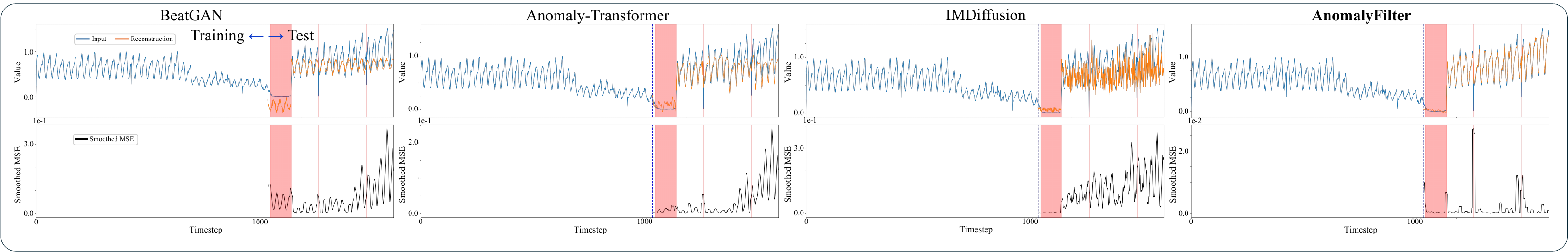} \\
    (b) Yahoo real A1real-28 \\
    \vspace{.5em}
    \includegraphics[width=1\linewidth]{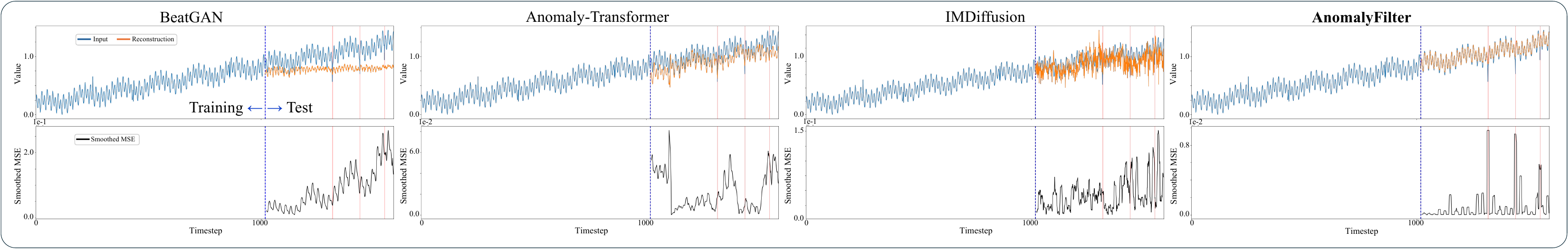} \\
    (c) Yahoo bench A3Benchmark-TS3 \\
    \vspace{.5em}
    \includegraphics[width=1\linewidth]{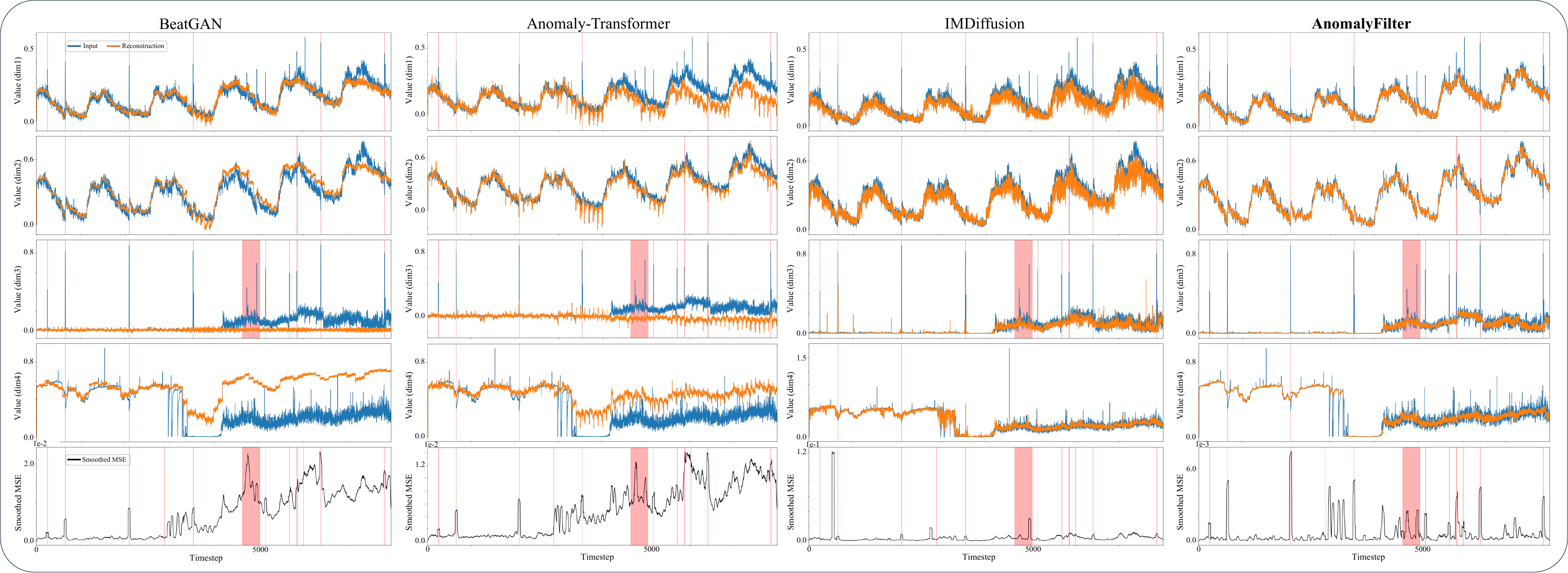} \\
    (d) SMD machine-1-6 \\
     \caption{
     Reconstruction examples across different methods of (a) AIOps, (b) Yahoo real, (c) Yahoo bench, and (d) SMD.
     }
    \label{fig:case_all}
\end{figure*}
We select six denoising-based methods to evaluate their scalability with respect to varying sample sizes during training and inference.
For diffusion models, we set forward and reverse diffusion steps to 50.
All methods are evaluated under the same NVIDIA RTX A6000.
As shown in \fig{\ref{fig:scalable}} (a), the training times are comparable across methods, except for \dddr.
During inference shown in \fig{\ref{fig:scalable}} (b), DAE and \dddr, which predict in a single step, are faster than the others.
Diffusion-based methods that require multiple reverse diffusion steps exhibit similar inference times.

\subsection{Case Studies} \label{sec:case}
We present further case studies in \fig{\ref{fig:case_all}} that visually show the difference between baselines and \method.

\subsubsection{AIOps KPI-301c}
This dataset consists of noisy sequences whose values fluctuate in a fine-grained manner, as shown in \fig{\ref{fig:case_all}} (a).
It includes different types of anomalies, such as a pattern anomaly at around timestep 1600 and many point anomalies. 
\method has a higher reconstruction error for all anomalies visualized in the figure, while accurately reconstructing normal parts.
In contrast, other methods fail to detect anomalies because their reconstruction error on normal parts is relatively larger than that of anomalous parts.

\subsubsection{Yahoo real A1real-28}
The Yahoo dataset provides limited training data, and some test data exhibit characteristics different from the training data.
As shown in \fig{\ref{fig:case_all}} (b), this sequence contains a cutoff anomaly shortly after the start of the test period (around timestep 1000), as well as two point anomalies.
Due to the limited amount of training data, all methods struggle to detect the cutoff anomaly.
Only \method has a high reconstruction error for the subsequent point anomalies.

\subsubsection{Yahoo bench A3Benchmark-TS3}
\fig{\ref{fig:case_all}} (c) is from Yahoo bench.
This dataset includes a trend component, causing the test values to fall outside the range of the training data.
\beatgan fails to shift the baseline of the reconstruction, resulting in a gradual drift away from the input.
Although \at and \imdiffusion handle the baseline shift of test data, their large reconstruction error leads to failures in anomaly detection.
In contrast, \method, due to its gradual denoising nature from the input, produces a reconstruction that closely matches the input and successfully detects the point anomalies.

\subsubsection{SMD machine-1-6}
We show the results of SMD on \fig{\ref{fig:case_all}} (d).
In the SMD dataset, each variable in the multivariate time series is individually labeled for anomalies; however, as visualized, the labeling accuracy is often low.
Therefore, we adopt a commonly used setting where a time point is considered anomalous if any variable exhibits an anomaly.
Because both \beatgan and \at reconstruct from the latent space, the reconstruction errors for dim3 and dim4 increase after timestep 4000, causing earlier anomalies to become less distinguishable.
Moreover, the reconstruction of dim2 for \beatgan and dim1 for \at is affected by dim3 and dim4, as these methods reconstruct the whole dimensions from the common latent vector.
In contrast, denoising-based methods, which denoise from the input, maintain low reconstruction errors for dim3 and dim4, thereby avoiding such drifts.
Since \method selectively denoises only the anomalous parts, it is less susceptible to the influence of anomalies in other variables.

\section{More Related Work}

\subsubsection{Reconstruction-based Method for TSAD}
Inspired by the success of deep learning, numerous deep learning methods for TSAD have been proposed, with reconstruction-based methods constituting a significant proportion~\cite{survey_vldb}.
Reconstruction-based methods aim to detect anomalous behavior in an unsupervised manner by comparing the original input with the reconstruction.
They rely on the hypothesis that a model trained mostly on normal samples excels in normal parts but struggles in anomalous parts.
Deep generative techniques, including VAE~\cite{donut,interfusion} and GAN~\cite{tanogan}, are often utilized to build reconstruction-based methods.
For example, LSTM was incorporated into VAE to learn temporal features~\cite{lstmvae,omnianomaly}.
BeatGAN~\cite{beatgan}, proposed for detecting anomalies in ECG beats, is a GAN-based method that uses time series warping for data augmentation.
However, VAE and GAN are usually difficult to train because they require a careful balance between the two networks. 
Transformer architecture~\cite{dual-tf,anomalytransformer,dcdetector,subtransformer} has also been explored.
TranAD~\cite{tranad} introduces attention mechanisms from transformer models and incorporates adversarial training to jointly enhance the accuracy of anomaly detection.
These encoder-decoder models aim to learn the latent representation for the entire time series for data reconstruction and can be interpreted through the information bottleneck (IB) principle, namely the trade-off between concise representation and reconstruction accuracy~\cite{ib}.
However, in anomaly detection, controlling the IB is challenging due to the absence of anomaly knowledge and labels, often leading to underfitting (i.e., prioritizing conciseness) or overfitting (i.e., prioritizing reconstruction).

\section{Limitations and Future Directions}

\subsubsection{High-dimentional data}
Current benchmarks commonly used for multivariate time series anomaly detection assign the same label to all variables.
While SMD labels individual variables, their quality is low, as we observed clearly normal variables labeled as anomalies.
Therefore, models must predict anomalies if any variable in the multivariate data exhibits abnormal behavior.

Such labels mean that high-dimensional data pose challenges for \method, as the model learns noise that is independent across variables.
This leads to reconstructions where only the anomalous parts of the affected variables are denoised.
With MSE used as the anomaly score, the accumulation of errors across variables makes detecting anomalies increasingly difficult as the number of variables grows.
On the other hand, when using encoder-decoder models, an anomaly in one variable affects the entire latent state, thereby influencing the reconstruction of all variables.
Although they cannot pinpoint the specific anomalous variable, they are more suitable for the current labels.
Accordingly, future research could focus on developing noise mechanisms that account for inter-variable dependencies.

\subsubsection{Anomaly contamination}
In unsupervised anomaly detection, it is common to assume that all or most of the training samples are normal, which is known as the normality assumption~\cite{mitigatecontamination}.
However, in real-world scenarios, the presence of anomaly samples in the training set (i.e., anomaly contamination) is inevitable~\cite{redlamp}.
Under this assumption, reconstruction-based methods become vulnerable to contamination, as their strong representational power can lead to the accurate reconstruction of even anomalous patterns, thereby reducing performance.
As with most of the reconstruction-based methods, we adopt the normality assumption, making \method susceptible to this issue.
To the best of our knowledge, no diffusion-based approach has explicitly addressed anomaly contamination.
Handling anomaly contamination within a diffusion-based framework remains an open and promising research direction.

\end{document}